\documentclass{article}

\usepackage[english]{babel}
\usepackage{natbib}
\usepackage{geometry}
\usepackage[utf8]{inputenc} 
\usepackage[T1]{fontenc}    
\usepackage{hyperref}
\hypersetup{
	colorlinks=true,
	citecolor=blue,
	linkcolor=magenta,
}       
\usepackage{url}            
\usepackage{booktabs}       
\usepackage{amsfonts}       
\usepackage{nicefrac} 
\usepackage{enumitem}
\usepackage{microtype}      
\usepackage{xcolor}         
\usepackage{dsfont}

\usepackage{amsmath}
\usepackage{amssymb}
\usepackage{mathtools}
\usepackage{amsthm}
\usepackage{comment}
\usepackage{subfig}
\usepackage{mathrsfs}
\usepackage{enumitem}
\usepackage{color}
\usepackage{derivative}
\usepackage{physics}

\usepackage{algorithm,algorithmic}

\theoremstyle{plain}
\newtheorem{theorem}{Theorem}[section]

\newtheorem{lemma}[theorem]{Lemma}
\newtheorem{corollary}[theorem]{Corollary}
\theoremstyle{definition}
\newtheorem{definition}[theorem]{Definition}
\newtheorem{assumption}{Assumption}
\theoremstyle{remark}

\usepackage{amsmath,amsfonts,bm}

\newcommand{\cL}{\mathcal{L}}

\newcommand{\cX}{\mathcal{X}}

\newcommand{\cS}{\mathcal{S}}
\newcommand{\cP}{\mathcal{P}}

\newcommand{\idc}{\mathds{1}}

\newcommand{\cls}{\mathrm{\langle cls \rangle}}

\newcommand{\hbE}{\widehat{\mathbb{E}}}
\newcommand{\sfmax}{\texttt{Softmax}}
\newcommand{\Diag}{\text{Diag}}
\newcommand{\err}{\text{err}}
\newcommand{\lag}{\mathtt{L}}

\def\1{\bm{1}}

\def\eps{{\epsilon}}

\def\mE{{\bm{E}}}

\def\mM{{\bm{M}}}

\DeclareMathAlphabet{\mathsfit}{\encodingdefault}{\sfdefault}{m}{sl}
\SetMathAlphabet{\mathsfit}{bold}{\encodingdefault}{\sfdefault}{bx}{n}

\newcommand{\R}{\mathbb{R}}

\DeclareMathOperator*{\argmax}{arg\,max}
\DeclareMathOperator*{\argmin}{arg\,min}

\title{Attention with Trained Embeddings  Provably Selects Important Tokens}

\author{Diyuan Wu$^{1,}$\thanks{Equal contribution.} \quad    Aleksandr Shevchenko$^{2,}$\footnotemark[1]  \quad Samet Oymak$^3$ \quad  Marco Mondelli$^1$ }

\date{}

\begin{document}
\maketitle

\footnotetext[1]{Institute of Science and Technology Austria (ISTA). Emails: \texttt{\{diyuan.wu, marco.mondelli\}@ist.ac.at}}
\footnotetext[2]{ETH Z\"urich. Email: \texttt{aleksandr.shevchenko@inf.ethz.ch}}
\footnotetext[3]{University of Michigan.  Email: \texttt{oymak@umich.edu}}

\begin{abstract}
Token embeddings play a crucial role in language modeling but, despite this practical relevance, their theoretical understanding remains limited. Our paper addresses the gap by characterizing the structure of embeddings obtained via gradient descent. Specifically, we consider a one-layer softmax attention model with a linear head for binary classification, i.e., $\sfmax( p^\top \mE_X^\top ) \mE_X v =  \frac{ \sum_{i=1}^T \exp(p^\top E_{x_i})  E_{x_i}^\top v}{\sum_{j=1}^T \exp(p^\top E_{x_{j}}) }$, where $\mE_X = [ E_{x_1} , \dots, E_{x_T} ]^\top$ contains the embeddings of the input sequence, $p$ is the embedding of the $\cls$ token and $v$ the output vector. First, we show that, already after a single step of gradient training with the logistic loss, the embeddings $\mE_X$ capture the importance of tokens in the dataset by aligning with the output vector $v$ proportionally to the frequency with which the corresponding tokens appear in the dataset.  Then, after training $p$ via gradient flow until convergence, the softmax selects the important tokens in the sentence (i.e., those that are predictive of the label), and the resulting $\cls$ embedding maximizes the margin for such a selection. Experiments on real-world datasets (IMDB, Yelp) exhibit a phenomenology close to that unveiled by our theory.
\end{abstract}

\section{Introduction}

The introduction of the attention mechanism \cite{bahdanau15translation,vaswani2017attention} marked a paradigm shift in the design of frontier machine learning models, leading to significant advances such as ChatGPT \cite{achiam2023gpt}, Claude \cite{anthropic2025claude}, AlphaFold \cite{jumper2021highly}, CLIP \cite{radford2021learning} and Dall-E \cite{ramesh2021zero}. This success prompted a surge of interest in understanding the structure and function of attention layers, with their optimization dynamics and inductive biases 
being object of extensive theoretical research \cite{abbe2024far,cabannes2024scaling,geshkovski2023emergence,makkuva2024attention,ataee2023max,vuckovic2020mathematical} (see also Section \ref{sec:related}).
Embeddings are a crucial component of the attention mechanism \cite{zhou2023one}, especially for downstream adaptation \cite{houlsby2019parameter, jiang2024repurposing, kossen2023three} with some works \cite{levine2024cell2sentence, zhou2023one} specifically highlighting their importance. However, despite the importance of learning embeddings, 
the existing analyses of transformer-like architectures either ignore the properties of embeddings by resorting to orthogonal structures \cite{yang2024training}, or omit embeddings completely by considering unprocessed inputs \cite{tiberi2024dissecting}.

Our paper fills this gap by studying directly the embedding training dynamics. Specifically, we aim to provide theoretical insight to the following questions:

\begin{quote}

\centering

\emph{    What is the structure learnt by the embeddings during gradient descent training?}

\emph{How is this structure related to the statistical properties of the data?}
\end{quote}
In Figure \ref{fig:imdb_corr_2layer}, we investigate these questions by analyzing the embeddings of a two-layer transformer trained on a sentiment analysis task on  IMDB and Yelp reviews. The plots reveal a remarkable simplicity in the structure of the learned embeddings, which capture the frequency of appearance of tokens in the dataset. Specifically, 
 the predictive mechanism (overlap with the regression coefficient $v$) favors the tokens which appear more frequently in the corresponding positive/negative context. A similar pattern emerges at the selection stage of the attention mechanism (overlap with the $\cls$ embedding $p$), i.e., more frequent tokens have a higher attention score. 

\begin{figure}[t]
  \centering
  \includegraphics[width=0.49\textwidth]{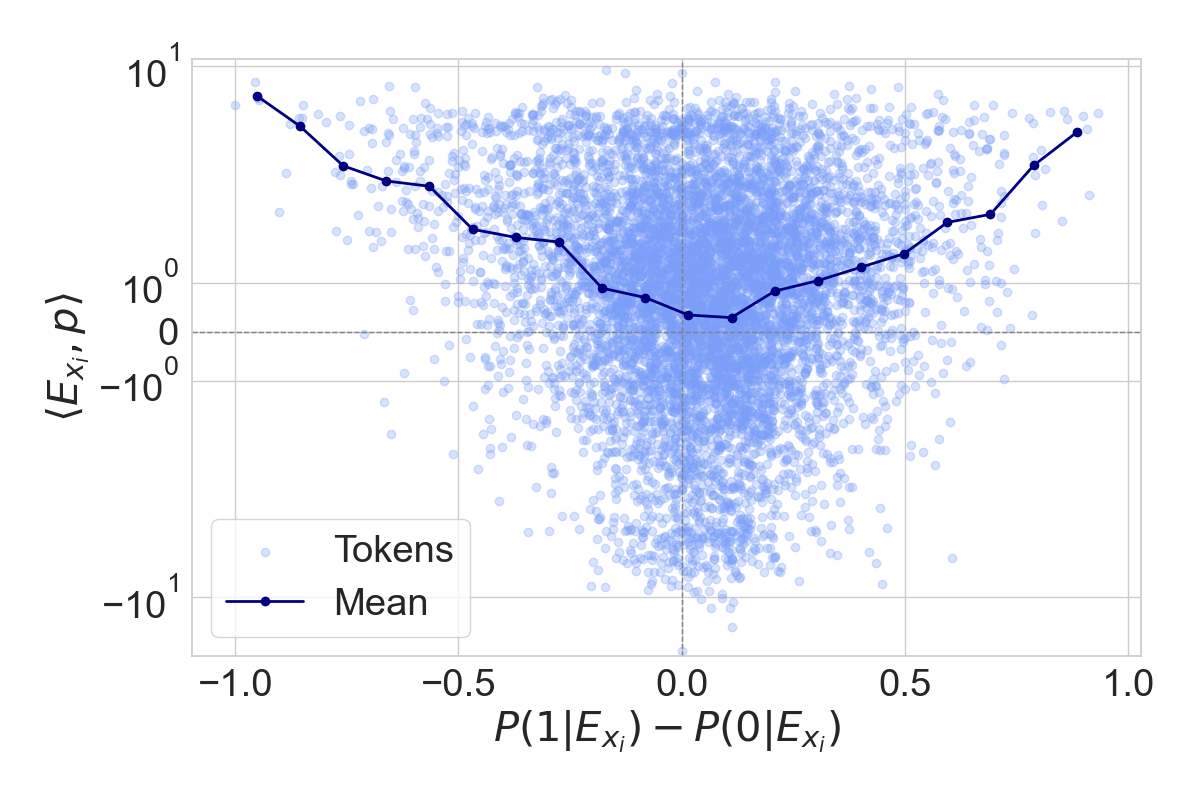}
  \includegraphics[width=0.49\textwidth]{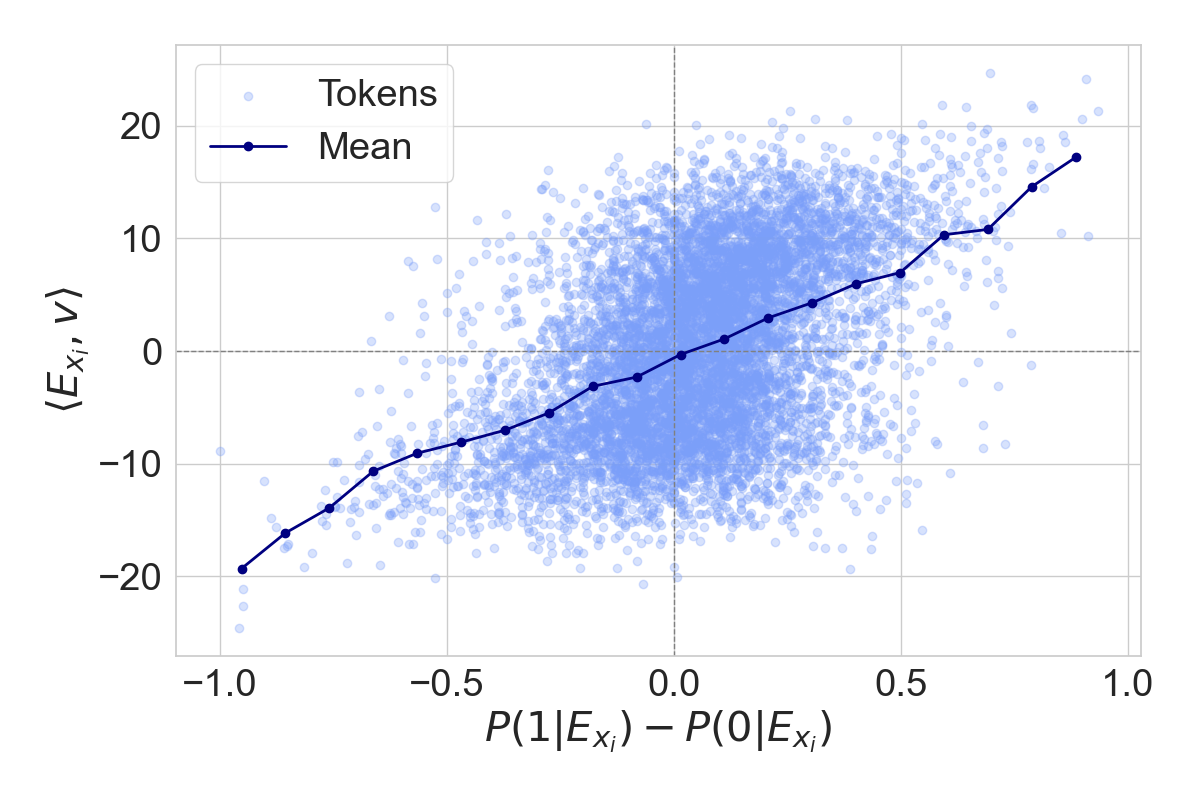}

  \includegraphics[width=0.49\textwidth]{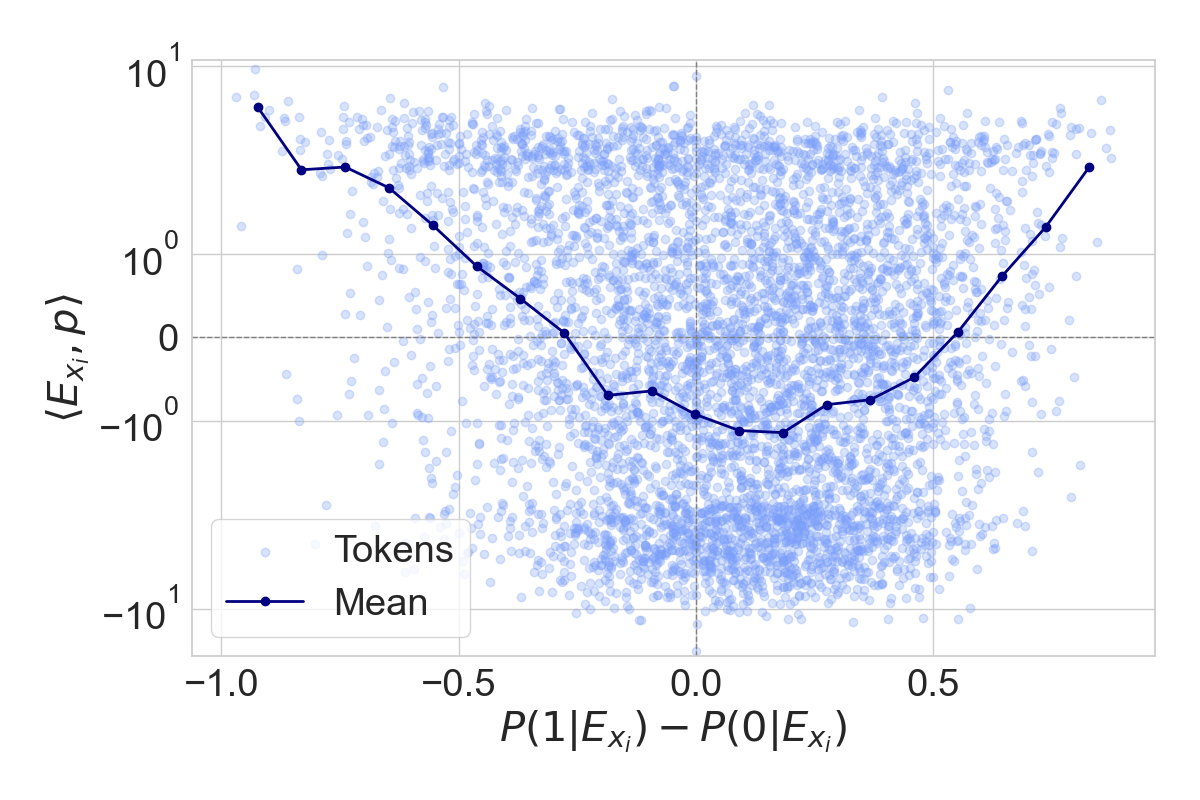}
  \includegraphics[width=0.49\textwidth]{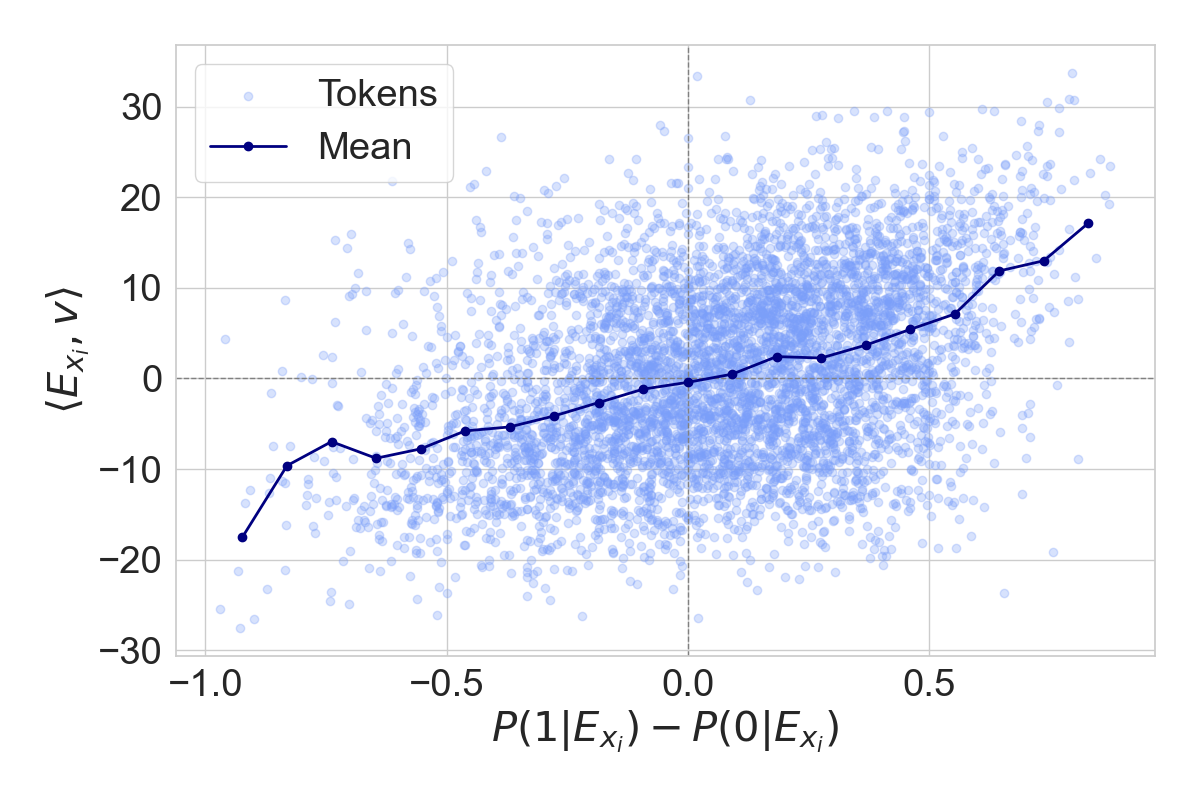}
  
  \caption{Dot-product of token embeddings with $\cls$ embedding $p$ (left) and regression coefficients $v$ (right), as a function of the token-wise difference in posterior probabilities, for IMDB (top row) and Yelp (bottom row) datasets. We consider the two-layer attention model in \eqref{eq:2layer_def} with all parameters trained until convergence.}
  \label{fig:imdb_corr_2layer}
\end{figure}

For the theoretical study of this emergent structure, we focus on a one-layer softmax attention model. Namely, for an input sequence $X = [x_1,\cdots,x_T]$, the output of the model is given by 
\begin{equation}\label{eq:f}
    f(X; p, \mE) =  \sfmax( p^\top \mE_X^\top ) \mE_X v =  \frac{ \sum_{i=1}^T \exp(p^\top E_{x_i})  E_{x_i}^\top v}{\sum_{j=1}^T \exp(p^\top E_{x_{j}}) },
\end{equation} 
where $\mE_X = [ E_{x_1} , \dots, E_{x_T} ]^\top$ contains the embeddings of the input $X$, $p$ is the embedding of the $\cls$ token and $v$ is the final regression vector. Our main results are summarized below:

\begin{itemize}    \item We show that, already after a single step of gradient training with the standard logistic loss, the embeddings $\mE_X$ capture the importance of tokens in the dataset by aligning with the output vector $v$ proportionally to the corresponding  empirical frequencies (Lemma \ref{lem:overlap_v}). 

    \item In a setting where each sequence contains a single important token, the $\cls$ embedding obtained from gradient flow must select all important tokens. We further characterize all the possible directions that the $\cls$ embedding may converge to, which are the max-margin solutions associated to feasible token selections (Theorem \ref{thm:complete_seq}).

    \item While in general the $\cls$ embedding may select irrelevant tokens, we identify sufficient conditions leading to the selection only of important tokens (Lemmas \ref{lem:no-select-all} and \ref{lem:suff}).

\end{itemize}

\section{Related work}\label{sec:related}

\paragraph{Implicit bias, margin maximization, attention.} The implicit bias literature has been instrumental in understanding the behavior of neural networks or overparameterized models optimized by gradient methods \cite{chizat2020implicit,arora2019implicit,neyshabur2014search}. A key phenomenon is that gradient descent on separable data with logistic loss directionally converges to the max-margin separator \cite{soudry2018implicit,ji2019implicit}. 
More recently, a series of works \cite{tarzanagh2023transformers,magen2024benign,julistiono2024optimizing,vasudeva2024implicit,li2024mechanics,sheen2024implicit,ataee2023max,sakamoto2024benign} has established an 
equivalence between the optimization geometry of self-attention and a hard-margin SVM problem selecting a subset of tokens via linear constraints on the outer-products of token pairs. Compared to these works that mostly focus on the training of single-layer attention weights, we point out two differences. First, we study the role of embeddings and their joint training with the $\cls$ token. Second, under our data model, we establish benign properties of the solution reached at convergence (which may not hold for arbitrary datasets \cite{ataee2023max}).

\paragraph{Theory of attention.} 
A line of work \cite{makkuva2024attention,makkuva2024local,nichani2024transformers} has explored whether attention-based architectures can extract causal structure from Markovian inputs. 
 The mechanics of 
next-token prediction when training a single self-attention layer is characterized in \cite{li2024mechanics}.
Towards understanding how to utilize structural properties of the data, the behavior of transformers on sparse token selection tasks is considered in  \cite{sanford2023representational,wang2024transformers}. 
The study \cite{ildiz2024self} provides a theoretical justification to the tendency of modern language models to generate repetitive text by showing that the underlying self-attention mechanism collapses into sampling only a limited subset of tokens. This stands in contrast to the slightly different setup of \cite{tian2023scan} where the transformer model does not degrade to a ``winner-takes-all'' strategy. The works \cite{geshkovski2023emergence,geshkovski2023mathematical,geshkovski2024dynamic} take a mean-field view to analyze the clustering behavior in transformer representations that emerges after successive applications of the attention block. 
Under a random feature design, it is shown in \cite{bombari2024towards} that softmax attention exhibits a sensitivity property which allows for a sharp change in attention scores given the perturbation of a single token. The role of the attention mechanism is also studied in \cite{oymak23attention} for prompt-tuning and in \cite{gozeten2025test} for test-time-training. 

\section{Problem setup}

\paragraph{Data model.} We focus on binary text classification problems. We consider a (context) vocabulary set $\mathcal{S}$ with size $|\mathcal{S}|$, together with a $\cls$ token for classification. Let $(X_i,y_i)_{i=1}^n$ be the dataset containing $n$ context sequences, where $y_i\in \{-1, 1\}$ and each context sequence $X\in \mathcal X_n:=\{X_1, \ldots, X_n\}$ contains $T$ tokens, i.e., $X = [x_1, \dots, x_T]$ with $x_i \in \cS$. Without loss of generality, we let $\cS$ be the set of tokens that appears in $\mathcal{X}_n$, as the embeddings of the remaining tokens are not trained and are not relevant for the problem at hand.

\paragraph{Architecture.} We consider a one-layer softmax attention model with a linear head for classification. First, we append a $\cls$ token at the end of the sequence $X$, and then we embed each token into a vector of dimension $d$. Namely, after the embedding layer, we have $\mE_X = [ E_{x_1} , \dots, E_{x_T} ]^\top \in \R^{T \times d}$, where $E_s \in \R^d$ denotes the embedding of the token $s.$
We let $\mE \in \R^{|\mathcal S| \times d}$ be the embedding matrix of all context tokens and $p \in \R^d$ the embedding of the $\cls$ token.

We focus on the architecture defined in \eqref{eq:f} 
where, given a vector $a\in \mathbb R^T$, $[\sfmax(a)]_i:=\frac{\exp(a_i)}{\sum_{j = 1}^T\exp( a_{j})}$ for $i\in \{1, \ldots, T\}$. We remark that the same model is also studied in \cite{ataee2023max,sakamoto2024benign}. In practice, it is common to include the $W_{KQ}$ matrix and consider a model with output $f(X; p,W_{KQ}, \mE) =  \sfmax( p^\top W_{KQ} \mE_X^\top ) \mE_X v. $ Since $p^\top W_{KQ}$ plays the same role as $p$ and one can easily reconstruct  $W_{KQ}$ from $p$ in each gradient update as discussed in \cite{ataee2023max}, we use the model in \eqref{eq:f} for simplicity.

\paragraph{Optimization problem.} The output vector $v$ is fixed and all the embedding vectors $p, \mE$ are trained with the standard logistic loss:
\begin{equation}\label{eq:loss}
    \cL(\mE,p) = \frac{1}{n}\sum_{k=1}^n  \log(1 + \exp(- y_k f(X_k;\mE,p)))=\hbE \left[ \log(1 + \exp(- y f(X;\mE,p))) \right],
\end{equation}  where the notation $\hbE$ is a shorthand for the average over the dataset $\mathcal D=\{(X_k, y_k)\}_{k=1}^n$. 

\paragraph{Empirical statistics of each token in the dataset.}

The goal of the paper is to characterize the structure of the embeddings $\mE,p$ obtained by optimizing the objective \eqref{eq:loss} via gradient descent, and we show that such structure is related to the empirical statistics of the tokens in the dataset. Specifically, after training, the softmax attention learns to select tokens that are more correlated to the labels based on the dataset.  To quantify the correlation between a token $s$ and the label $y$, we define the \emph{average signed frequency} of a token as:
\begin{equation}\label{eq:alphas}
    \alpha_s := \frac{1}{nT}\sum_{(X, y)\in\mathcal D}y \sum_{i=1}^T\idc_{x_i =s} =\frac{1}{T}\hbE\left[y \sum_{i=1}^T\idc_{x_i =s}\right].
\end{equation} 
In words, $\alpha_s$ is obtained by taking the number of occurrences of $s$ in sequences with a positive label, subtracting the number of occurrences of $s$ in sequences with a negative label, and finally dividing by the total number of tokens $nT$. As such, it provides an average of the signed frequency of $s$, where the sign comes from the label of the sequences in which the token appears. 

\begin{definition}[Positive, negative and irrelevant tokens]
    We say that a token $s$ is \emph{(i)}  \textit{positive} if $\alpha_s > 0$, \emph{(ii)} \textit{negative} if $\alpha_s <0$, and \emph{(iii)} \textit{irrelevant} if $\alpha_s = 0$. Moreover, we say that a token $s$ is \textit{completely positive} if it appears only in sequences with label $1$, and \textit{completely negative} if it appears only in sequences with label $-1$.
\end{definition}

In words, a token is positive (negative) when it is more frequently associated to the positive (negative) label; tokens that appear the same number of times associated to positive and negative labels are irrelevant. The quantity $\alpha_s$ quantifies how positive/negative a token is. Intuitively, if either $\alpha_s>\alpha_{s'}>0$ or $\alpha_s<\alpha_{s'}<0$ , then the token $s$ is more relevant than the token $s'$ for the classification task and, therefore, we expect that this will be reflected into the structure of the corresponding embeddings. 


\section{Main theoretical results}
\label{sec:1step_1impt}

We show that the trained attention model \eqref{eq:f} learns to select important tokens from the empirical statistics of the data. First, in Section \ref{subsec:embed} we prove that training the context embeddings with a single gradient step suffices to capture the empirical importance of tokens in the dataset; then, in Section \ref{subsec:select} we characterize the implicit bias of training the $\cls$ embedding $p$ until convergence, having fixed the context embeddings after the first gradient step.

\subsection{One step of gradient descent learns the importance of the tokens}\label{subsec:embed}

 We start by showing that the first step of gradient descent is already enough to give a correlation between embeddings and the output vector $v$. 
 Furthermore, this correlation is proportional to the average signed frequency defined in \eqref{eq:alphas}. We initialize $v$ with any unit-norm vector and $E_s^0, p^0 \overset{\text{i.i.d.}}{\sim} \mathcal{N}(0, \frac{1}{d} I)$ for all $s\in\mathcal S$. Then, we
perform one step of gradient descent with step size $\eta_0$ on all trainable embeddings:
\begin{equation}\label{eq:GD1}
   \begin{split}
       p^1 &= p^0 - \eta_0  \nabla_{p} \cL(\mE^0,p^0),\qquad  E_s^1 = E_s^0 - \eta_0  \nabla_{E_s} \cL(\mE^0,p^0), \quad \mbox{for all } s \in \cS.
   \end{split}
   \end{equation}
   

\begin{lemma}
\label{lem:overlap_v}
 For any $\delta > 0$, let 
\begin{equation}
     d \geq \max\left\{256, \left(2 \log\frac{|\cS|^2}{\delta} \right)^2\right\}.
 \end{equation} 
 Then, after the first step of gradient descent in \eqref{eq:GD1}, we have that, for any $s\in\mathcal S$, 
\begin{equation}
 \label{eq:embd_1step}
    E_s^1 = E_s^0 + \frac{\eta_0}{2} \alpha_s v + \err_s, \qquad p^1 = p^0 + \err_p,
    \end{equation} where the error terms $\err_s, \err_p$ are bounded with probability at least $1-\delta$ as 
\begin{equation}
    \max\{\max_{s \in \cS} \|\err_s\|_2, \|\err_p \|_2\} \leq 11\eta_0  d^{-\frac{1}{4}}.  
\end{equation}
\end{lemma}

Lemma \ref{lem:overlap_v} implies that after one step of training, the embedding vector $E_s$ of each token $s$ learns the empirical importance of the tokens by adding a vector in the direction of the output vector $v$ with magnitude  proportional to $\alpha_s$.

The proof follows from the structure of the gradient update. In particular, it can be shown that 
\begin{equation}    \label{eq:charact}
\nabla_{E_s} \cL(\mE^0,p^0)  = - \hbE\left[y g(X,y) \left( \sum_{i=1}^T  (\sum_{j\neq i} ( \idc_{x_i=s} - \idc_{x_j=s}) q_{i} q_{j} ) (E_{x_i}^0)^\top v p^0 + \sum_{i=1}^T \idc_{x_i = s} q_{i} v \right) \right],
\end{equation}
with $q_i := \frac{\exp((p^0)^\top E_{x_i}^0) }{ \sum_{j=1}^T \exp((p^0)^\top E_{x_j}^0)}$. Note that, for all $x_i$, $(p^0)^\top E_{x_i}^0$ is of order $1/\sqrt{d}$, due to the independent Gaussian initialization. Thus,  the characterization in \eqref{eq:charact} implies that the gradient is roughly $\frac{\eta_0}{2} \alpha_s v$ plus a term that is vanishing in $d$. 
The full argument is deferred to Appendix \ref{apx:pf_lem_overlap_v}.



\subsection{Gradient flow on $p$ performs max-margin token selection}\label{subsec:select}

Lemma \ref{lem:overlap_v} shows the informative overlap between the output vector $v$ and the context embedding vectors $E_s$ after the first step of gradient descent. However, \eqref{eq:embd_1step} also implies that the overlap between the $\cls$ embedding vector $p$ and  $E_s$ does not improve after the first step. Thus, next, we study the training dynamics of $p$, characterizing its implicit bias. Specifically, we fix the context embedding matrix to $\mE^1$ (obtained after the first gradient step) and train the $\cls$ embedding vector $p$ with gradient flow initialized at $p^1$ (obtained after the first gradient step): 
\begin{equation}
\label{eq:gf_p}
    \frac{\dd}{\dd t} p_t =  -  \nabla_{p} \cL(\mE^1,p_t).
\end{equation} 

We consider gradient flow for technical convenience, and all results in this section can be readily extended to gradient descent with small enough step size. For the rest of the section and in the related proofs appearing in the appendix, we will refer to the embeddings in $\mE^1$ as $E_s$ and not $E_s^1$, omitting the superscript to favor readability. 


\paragraph{Max-margin token selection.} 
Given the $\cls$ embedding $p$ and a sequence $X$, we denote the set of tokens in $X$ selected by $p$ as 
\begin{equation}
    \cS_{X}(p) = \{ s : s = \argmax_{s \in X} p^\top E_s \},
\end{equation} and we define $\overline{\cS_{X}(p)} = X \setminus \cS_{X}(p)$. Intuitively, given a sequence $X,$ the selected tokens in $X$ have the largest softmax weight (proportional to $\exp(p^\top E_{x_i})$). Note that, for $p' \neq p,$ we may have that $\cS_X(p') = \cS_X(p)$ for all $X$. Thus, we 
define the equivalence relation 
\begin{equation}
    p \approxeq p' \Longleftrightarrow \cS_X(p) = \cS_X(p'),\qquad  \text{for all $X\in\mathcal X_n$}.
\end{equation} 
Intuitively, two vectors $p,p'$ are equivalent under the above relation if they select the same tokens for all the sequences.
Given a vector $p_\circ,$ we denote by $\cP_{p_\circ}$ its equivalence class, and we  define the set of max-margin directions among all vectors in $\cP_{p_\circ}$ as 
\begin{equation}
\label{eqn:max-margin}
    \begin{split}
        \cP_*(p_\circ) = &\left\{ \frac{\hat{p}}{\|\hat{p}\|_2}: \hat{p} = \argmin_{p \in \cP_{p_\circ}} \|p\|_2 \right. \\
        &\left. \text{s.t.} \quad  p^\top(E_s - E_{s'}) \geq 1,\quad \forall s \in \cS_X(\cP_{p_\circ}), \,\,\forall s' \in \overline{\cS_X(\cP_{p_\circ})}, \,\,\forall X \in \cX_n \right\}.
    \end{split}
\end{equation}

We first show in Lemma \ref{lem:sol_mm} (proved in Appendix \ref{apx:pf_lem_sol_mm}) below that the max-margin problem in \eqref{eqn:max-margin} always has a unique solution, which means that $\cP_*(p_\circ)$ is always a singleton. Thus, later on, we will use $\hat{p}(p_\circ)$ as the solution to \eqref{eqn:max-margin}, and $p_*(p_\circ) = \frac{\hat{p}(p_\circ)}{\|\hat{p}(p_\circ)\|_2}$. We drop the dependency on $p_\circ$ when there is no confusion. 

\vspace{-.3em}

\begin{lemma}
\label{lem:sol_mm}
     For any $p_\circ \neq 0,$ the max margin problem in \eqref{eqn:max-margin} has a unique solution denoted as $\hat{p}$. 
\end{lemma}






\paragraph{Implicit bias of gradient flow.} 
While Lemma \ref{lem:overlap_v} holds for any dataset, we need an extra assumption on the data to analyze the gradient flow, due to the complex loss landscape caused by softmax attention. 

\begin{assumption}
\label{asm:complete_token}
    Each sequence in $\mathcal{X}_n$ contains either a \emph{single completely positive} token or a \emph{single completely negative} token, and all remaining tokens are \emph{irrelevant}.
\end{assumption}

Assumption \ref{asm:complete_token} implies that all sequences in the dataset contain precisely one relevant token, and the relevant token also aligns with the label. We remark that 
datasets containing only one relevant token have been also considered in prior work, see
\citep[Theorem 1]{tarzanagh2023transformers} and \citep{marion2025attention}.
We further denote by $\cS_{c}$ the set containing all completely positive and all completely negative tokens.


\begin{theorem}
\label{thm:complete_seq}
    Under Assumption  \ref{asm:complete_token}, for any $\delta > 0,$ let  \begin{equation}\label{eq:condeta0}
         \eta_0 \geq 4 n^2 T^2, \quad d \geq \max\left\{ 256, \left(2 \log\frac{|\cS|^2}{\delta} \right)^2, (88 \eta_0^2+111\eta_0+2)^8 , |\cS|+3  \right\}.
    \end{equation} Let $p_t$ be the solution of the gradient flow \eqref{eq:gf_p}. Then, with probability at least $1-\delta$, we have that $\|p_t\|_2 \rightarrow \infty$. Furthermore, assuming that $p_\infty:=\lim_{t \rightarrow +\infty} \frac{p_t}{\|p_t\|_2}$ exists, the limiting direction $p_\infty$ satisfies the following properties with probability at least $1-\delta$: \begin{enumerate}            \item $p_\infty$ selects all completely positive and completely negative tokens, i.e., $\cS_{c} \subseteq \bigcup_X \cS_X(p_\infty)$.
            \item $p_\infty$ is the max-margin direction for such a selection, i.e., $p_\infty = p_*(p_\infty)$.
        \end{enumerate}
\end{theorem}


Theorem \ref{thm:complete_seq} shows that, if $p_t$ converges in direction, it must converge to the max-margin direction that selects all the completely positive/negative token. A sketch of the argument is given below and the complete proof is in Appendix \ref{apx:pf_complete_seq}.
\begin{proof}[Proof sketch]
    We prove the three statements separately. First, we show that $\|p_t\|_2\to\infty$ (Lemma \ref{lem:1_impt_norm}). To do so, we explicitly construct a vector $\hat{p}$ such that $\hat{p} \nabla_p \cL(\mE^1, p) < 0$ for all $p$. This means that there is no stationary point with finite norm, which implies that the norm of the vector obtained via gradient flow diverges.

    Next, we show that, if the directional limit $p_\infty$ exists, then it must select all the important tokens (Lemma \ref{lem:1_impt_norm2}). To do so, we note that, after one step, the model approximately selects the important tokens. This implies that $p_t$ selects important tokens for all $t$, as gradient flow cannot increase the loss. 
    
    Finally, we show that $p_\infty$ is the max-margin solution of a feasible selection (Lemma \ref{lem:max-margin}). To do so, we assume by contradiction that the directional limit is any vector $p'$ that is not the max-margin solution of a feasible selection. Then under this assumption, we prove that $\lim_{t \rightarrow \infty} \frac{p_t}{\|p_t\|_2} = p_*(p'),$ which is in fact the solution of \eqref{eqn:max-margin}, thus giving a contradiction. 
\end{proof}

Before proceeding with the characterization of the max-margin solution, we highlight some differences with respect to the related work \citep{ataee2023max}. Theorem 3 in \citep{ataee2023max} shows that gradient descent on $p$ converges to a locally optimal max-margin solution when initialized in a regime close enough to such solution, and Theorem 4 in \citep{ataee2023max} shows that the regularization path can only converge to locally max-margin solutions. However, these results do not exclude the possibility of the gradient flow converging to directions that are \emph{not} locally optimal and \emph{not} the max-margin direction. In contrast, we characterize all possible directions the gradient flow converges to,
showing that these are max-margin directions that select all completely positive/negative tokens. Furthermore, we do so without starting from an initialization that is close enough to such solution. This requires a different proof strategy as compared to \citep{ataee2023max}.

\paragraph{Characterization of the max-margin solution.} Theorem \ref{thm:complete_seq} gives that, if gradient flow converges in direction, the limiting direction is the max-margin one that selects all important tokens in each sentence. However, this does not exclude a-priori the possibility that gradient flow also selects some irrelevant tokens. To address this point, we now establish if and how many irrelevant tokens can be selected: in general, it is not possible to select all irrelevant tokens (Lemma \ref{lem:no-select-all}) and, under an additional assumption, no irrelevant token is selected (Lemma \ref{lem:suff}). 

\begin{lemma}
\label{lem:no-select-all}
    Suppose  $\hat{p}$ selects all the tokens, i.e., $\cS_X(\hat{p}) = \cS$. Then, $p_\infty \neq \hat{p}.$
\end{lemma}
\begin{proof}
If $\hat{p}$ selects all the tokens, then $\hat{p}^\top E_{x_i} = \hat{p}^\top E_{x_j}$ for all $x_i, x_j \in X$ and for all $X\in \mathcal X_n$. Thus, by Lemma \ref{lem:dir_grad}, $\hat{p}^\top \nabla_p \cL(\mE, p) = 0$ for any $p$, which gives the desired result.
\end{proof}

Lemma \ref{lem:no-select-all} shows that the directional limit $p_\infty$ (when it exists) cannot select all tokens and, as it selects all important ones, it must be biased towards them. As an application, consider the case where there is only one irrelevant token in the vocabulary. Then, the combination of Theorem \ref{thm:complete_seq} and Lemma \ref{lem:no-select-all} gives that only the completely  positive/negative tokens are selected by gradient flow. 

Going beyond the case where there is a single irrelevant token, the result below provides a sufficient condition for gradient flow to select only important tokens. 


\begin{lemma}
\label{lem:suff}
    Under Assumption \ref{asm:complete_token}, for any $\delta>0$, assume that \eqref{eq:condeta0} holds. Let $\hat{p}$ be the solution of the max-margin problem \eqref{eqn:max-margin} that only selects the completely positive/negative tokens, i.e., \begin{equation*}
        \begin{split}
            \hat{p} = &\argmin_p \|p\|_2, \qquad \text{ s.t. }\quad p^\top (E_{s_*^X} - E_s) \geq 1,\qquad  \forall s \in X \setminus  \{s_*^X\}, \,\,\forall X\in \mathcal X_n,
        \end{split}
    \end{equation*} 
  
    where $s_*^X$ denotes the unique completely positive/negative token in the sequence $X$.
    Assume that $p_\infty := \lim_{t \rightarrow \infty} \frac{p_t}{\|p_t\|_2}$ exists and that, for any $\hat{p}'$ solving \eqref{eqn:max-margin} with a different selection,  $\|\hat{p}\|_2 < (1-\mu)\|\hat{p}'\|_2$ for some constant $\mu$ that does not depend $d$. Then, by further taking $$d \geq \left\{ 
 \left( \frac{2816 Tn^2  (1+2\eta_0)}{\mu }   \right)^4, \left(2 \eta_0 \left( n \sqrt{2 \log \frac{|\cS|^2}{\delta}} + 44n\right) \right)^4 \right\},$$ 
    we have that $p_\infty = \frac{\hat{p}}{\|\hat{p}\|_2}$
    with probability at least $1-\delta$. 
\end{lemma}

The proof of Lemma \ref{lem:suff} is deferred to Appendix \ref{apx:pf_lem_suff}.
The sufficient condition of the result above requires the max-margin direction that does not select irrelevant tokens to have a larger margin than any other max-margin solution associated to a different token selection. We expect this to be the case e.g.\ for datasets where all the completely positive/negative tokens have the same $\alpha_s$. 
In fact, given the structure of the context embeddings in \eqref{eq:embd_1step}, the max-margin solution $\hat{p}$ is expected to satisfy $\hat{p}^\top v\approx 0, \hat{p}^\top E_s \approx 1, \hat{p}^\top E_{s'} \approx 0$ for all $s \in \cS_X(\hat{p})$ and $s' \in \overline{\cS_X(\hat{p})}$. Since the token embeddings at initialization are approximately orthogonal to each other, $\hat{p} \approx \sum_{s \in \cS_X(p)} E_s^0$, meaning that $\|\hat{p}\|_2 \approx \sqrt{|\cS_X(p)|},$ which implies that the sufficient condition holds.  

\section{Numerical experiments}
\begin{figure}[tbp]
  \centering
  \includegraphics[width=0.49\textwidth]{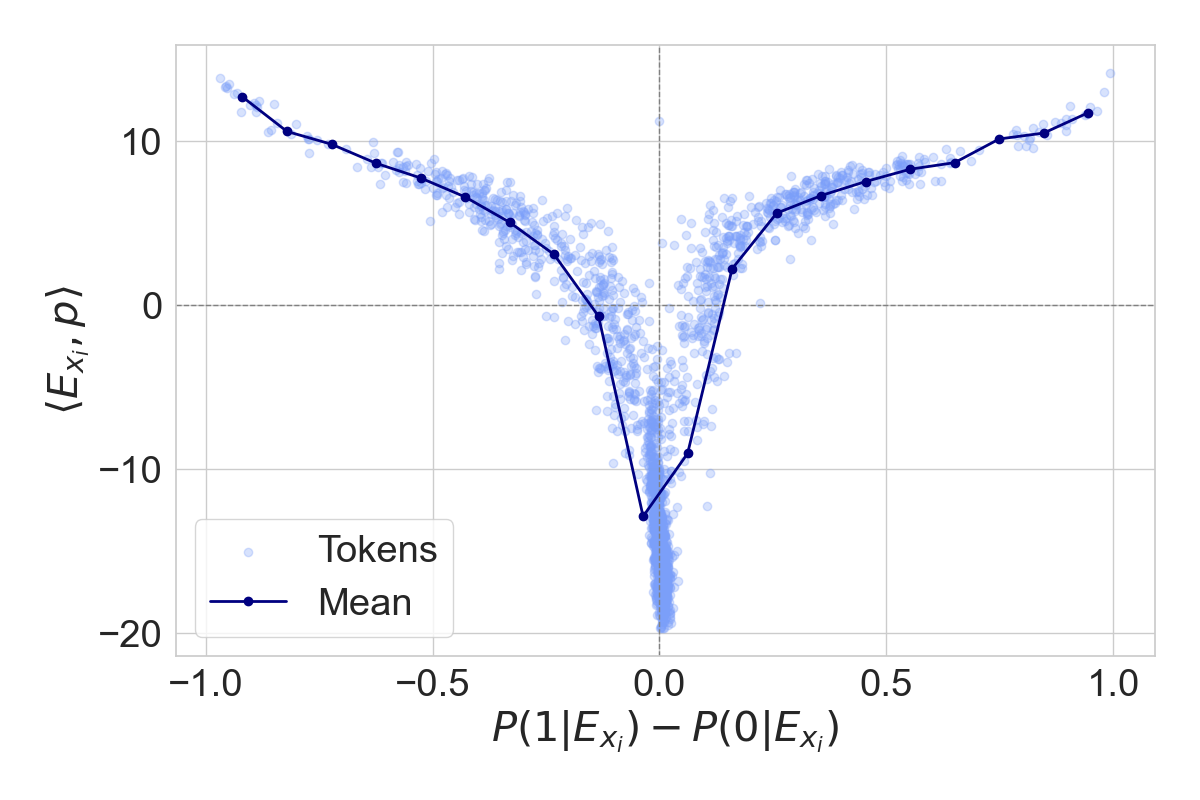}
  \includegraphics[width=0.49\textwidth]{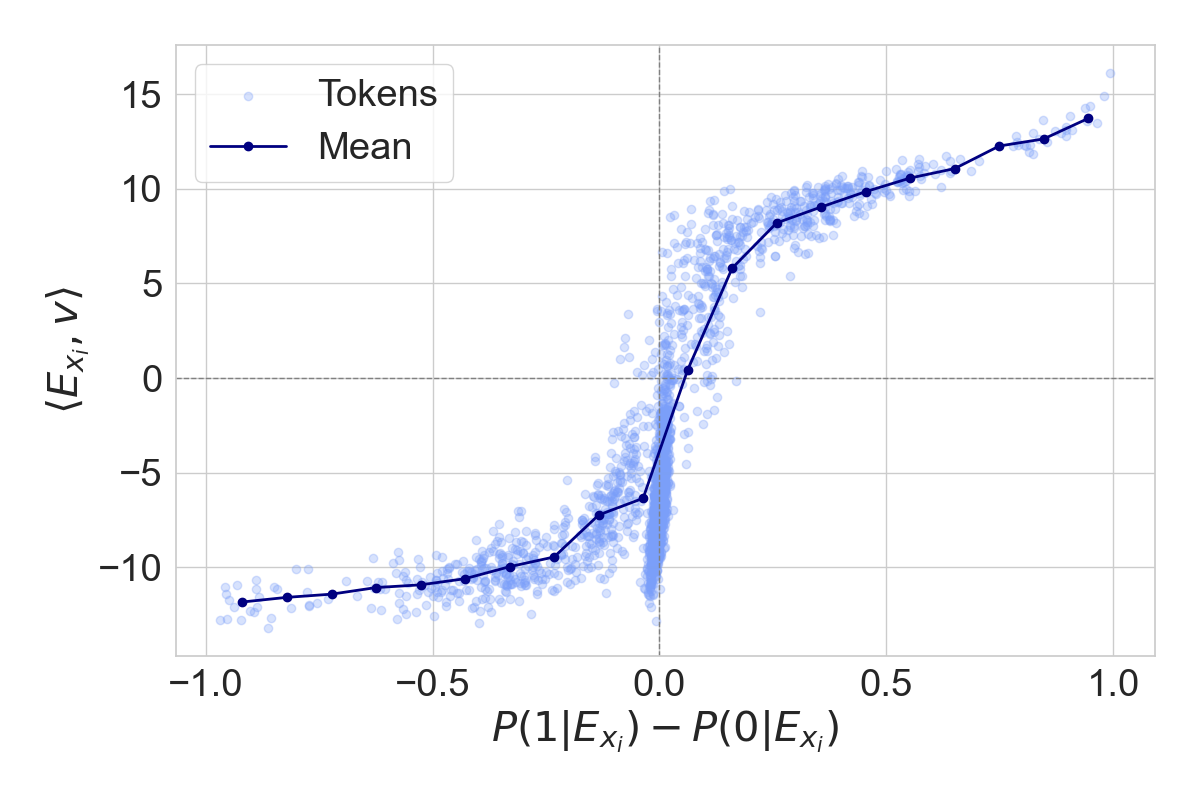}
  \caption{Dot-product of token embeddings with $\cls$ embedding $p$ (left) and regression coefficients $v$ (right), as a function of the token-wise difference in posterior probabilities for synthetic data sampled according to \eqref{eq:synth_data_exp}. We consider the one-layer attention model in \eqref{eq:f} with all parameters trained until convergence. The point cloud around zero corresponds to the tokens in the irrelevant set.}
  \label{fig:syn_corr}
\end{figure}

To support our theoretical findings, we showcase the correlation of the embeddings with the $\cls$ embedding $p$ and the output vector $v$, having trained \emph{all} the parameters with gradient descent until convergence. We consider different datasets (synthetic data in Figure \ref{fig:syn_corr}; IMDB/Yelp datasets in Figures \ref{fig:imdb_corr_2layer} and \ref{fig:imdb_corr}) and different architectures (one-layer model \eqref{eq:f} in Figures \ref{fig:syn_corr} and \ref{fig:imdb_corr}; two-layer model \eqref{eq:2layer_def} in Figure \ref{fig:imdb_corr_2layer}). Taken together, the experiments display an excellent agreement with our theory going beyond the one-layer architecture \eqref{eq:f} and also beyond the requirements on the data-generating process.
Specifically, the trained embeddings capture the importance of the corresponding tokens: the dot-product with $v$ is proportional to how frequently the token appears in positive sequences rather than in negative ones, and the dot-product with  $p$ is proportional to the modulus of such frequency. 
We detail below the experimental design.

\textbf{Synthetic data.} Let us define the data-generating process for the synthetic experiments in Figure \ref{fig:syn_corr}. The data is generated according to a $K$-level model. Namely, the vocabulary set $\mathcal{S}$ is partitioned as 
\begin{equation}\label{eq:synth_partition}
    \mathcal{S} = \tilde{\mathcal{S}} \cup \left\{\mathcal{S}^{-1}_k\right\}_{k=1}^K \cup \left\{\mathcal{S}^{+1}_k\right\}_{k=1}^K.
\end{equation}
Here, $\tilde{\mathcal S}$ contains \emph{irrelevant} tokens appearing in both positive and negative contexts with equal probability, while $\mathcal{S}^{+1}_k$ and $\mathcal{S}^{-1}_k$ (for $k\in \{1, \ldots, K\})$ contain tokens appearing \emph{mostly} in positive and negative contexts, respectively. Formally, define the importance levels $\tilde{\delta},\delta_1,\dots,\delta_K>0$. Then, given the sequence label $y \in \{-1,+1\}$ and $s\in\mathcal{S}$, we sample the tokens from the vocabulary as 
\begin{equation}\label{eq:synth_data_exp}
    p(s|y) = \begin{cases}
        \frac{1-\tilde{\delta}}{|\tilde{\mathcal{S}}|}, & s\in \tilde{\mathcal{S}},\\
        \frac{\tilde{\delta}(1-\delta_{k})}{\sum_{k=1}^{K}|\mathcal{S}_k^{y}|}, & s \in \mathcal{S}_k^{y}, \\
        \frac{\tilde{\delta}\delta_k}{\sum_{k=1}^{K}|\mathcal{S}_k^{\neg y}|}, & s \in \mathcal{S}_k^{\neg y},
    \end{cases}
\end{equation}
where $\neg$ denotes the binary inversion, i.e., $\neg (+1) = -1$ and $\neg (-1) = +1$. The law \eqref{eq:synth_data_exp} implies the following posterior distribution:
\begin{equation}\label{eq:synth_posterior}
    p(y|s) = \begin{cases}
        1/2, & s \in \tilde{\mathcal{S}}, \\
        1 - \delta_k, & s \in \mathcal{S}_k^{y}, \\
        \delta_k, & s \in \mathcal{S}_{k}^{\neg y}.
    \end{cases}
\end{equation}
From \eqref{eq:synth_posterior}, it is clear that \emph{(i)} $\tilde{\mathcal{S}}$ contains \emph{irrelevant} tokens as the posterior is uniform, and \emph{(ii)} $\delta_k$ quantifies the importance of the tokens in $\mathcal{S}_k^{\pm 1}$ by skewing the posterior to be $(\delta_k, 1-\delta_k)$. 
For the experiments in Figure \ref{fig:syn_corr}, we select the following hyper-parameters:
$|\mathcal{S}| = 2048$, $K=8$ and sequence length $T=256$;
$|\mathcal{S}_k^{+1}| = |\mathcal{S}_k^{-1}| $ with $|\mathcal{S}_k^{+1}| = 4 + 2^{k-1}$, and $|\tilde{\mathcal{S}}| = 964$;
$\tilde{\delta} = 0.05$ and $\{\delta_k\}_{k=1}^K = \{0.45, 0.35, 0.3, 0.25, 0.2, 0.1,0.05, 0.02\}$.

Figure \ref{fig:syn_corr} shows a clear separation between positive and negative tokens (right plot with the dot-product $\langle E_{x_i}, v \rangle$), and the selection mechanism ($\cls$ token) assigns high weights to tokens that have larger importance $\delta_k$ (left plot with the dot-product $\langle E_{x_i}, p \rangle$).

\begin{figure}[t]
  \centering
  \includegraphics[width=0.49\textwidth]{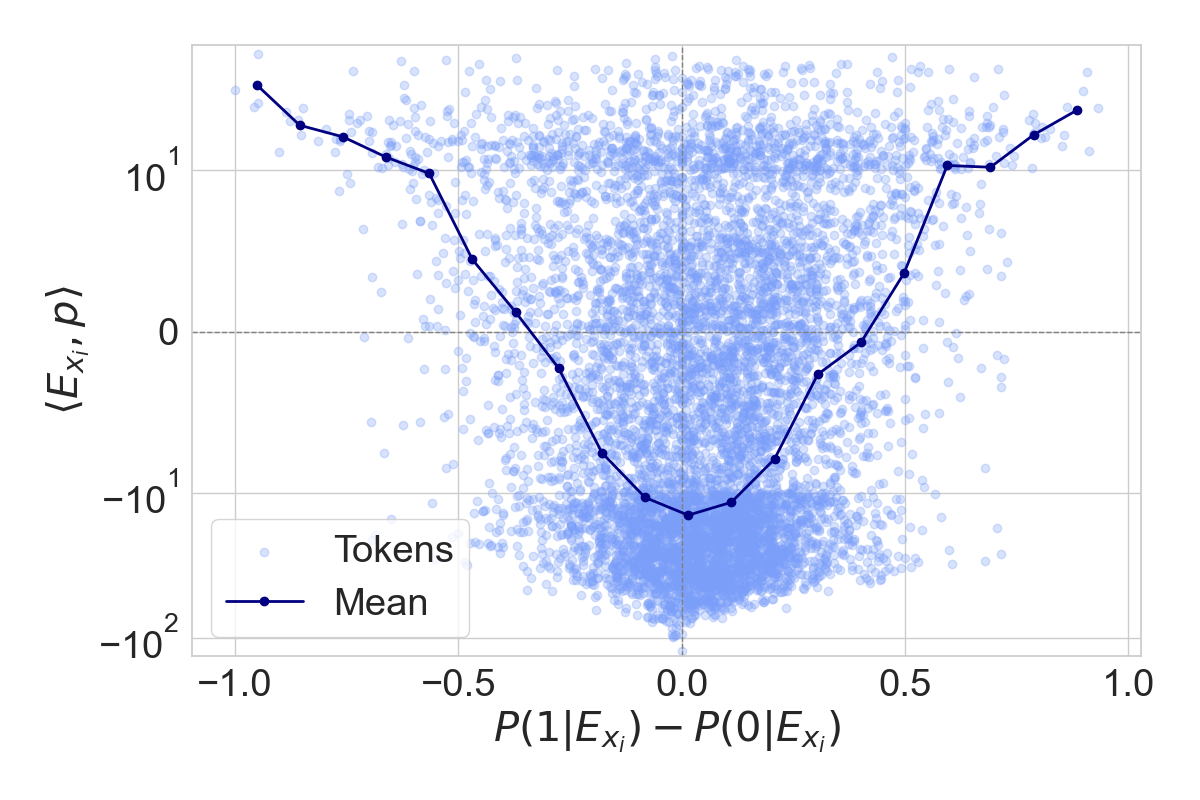}
  \includegraphics[width=0.49\textwidth]{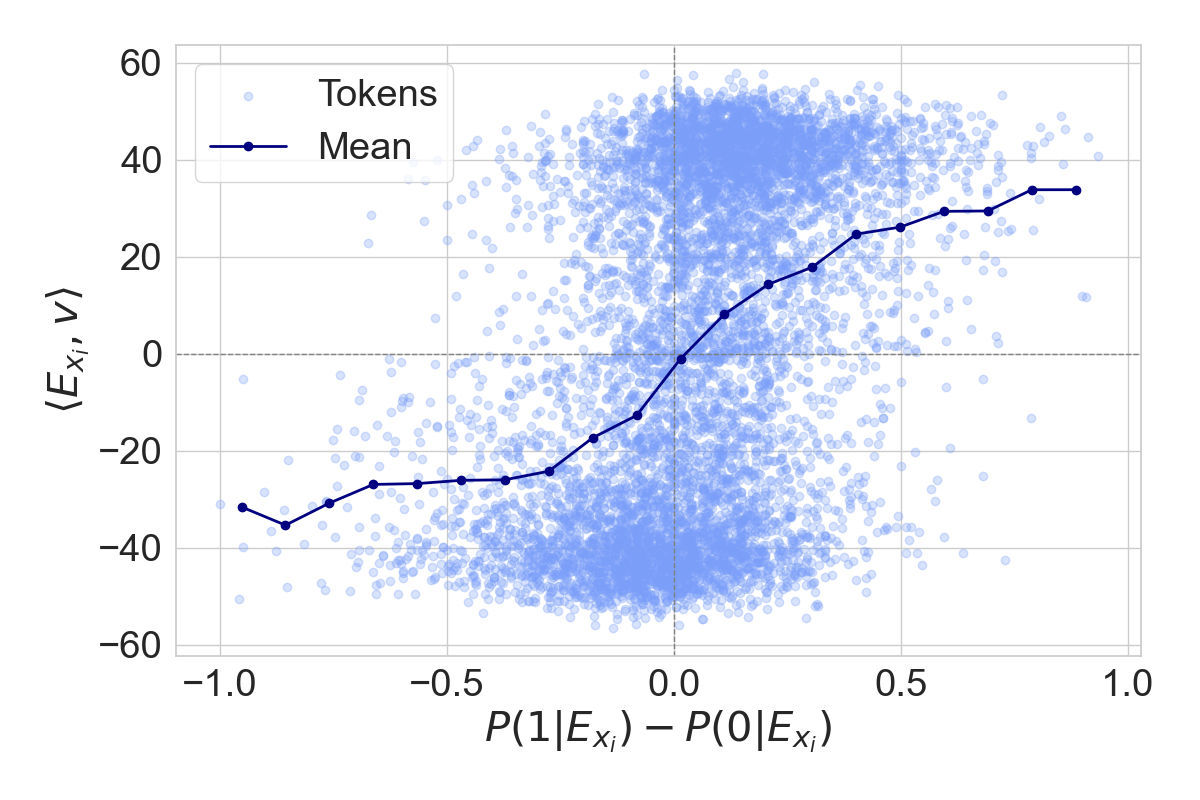}

  \includegraphics[width=0.49\textwidth]{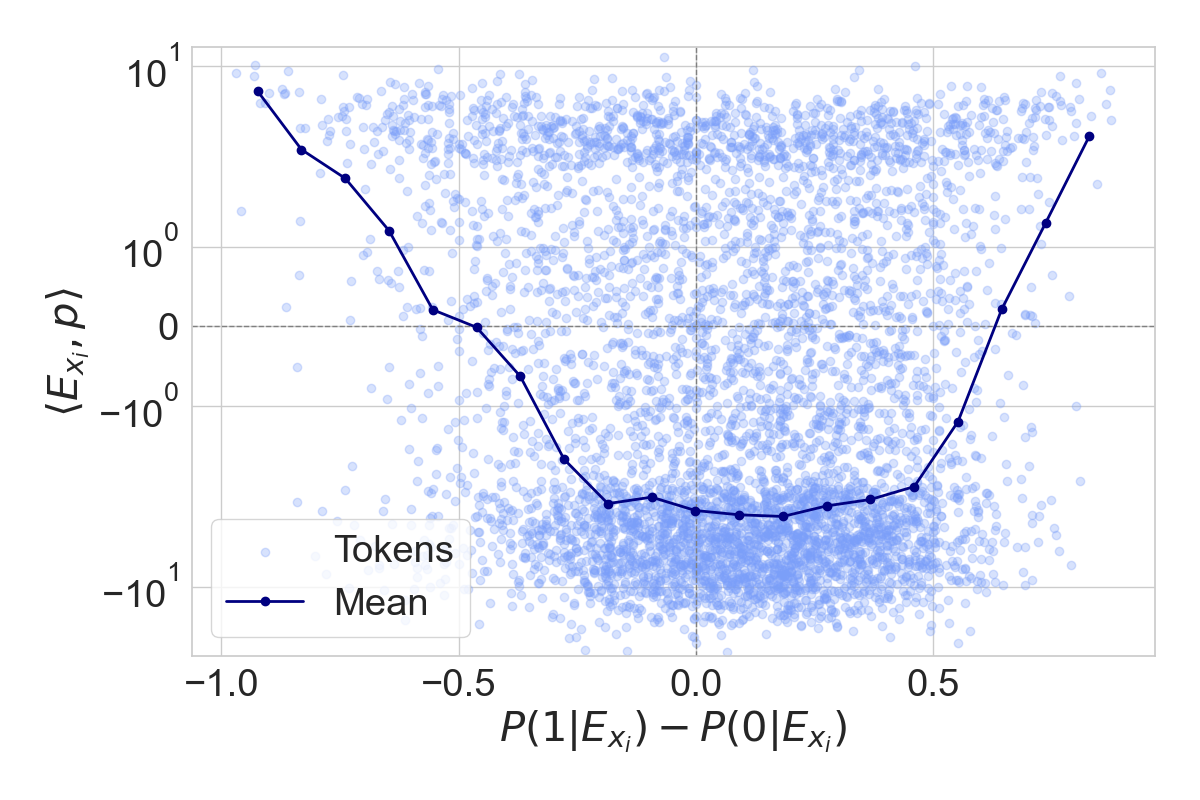}
  \includegraphics[width=0.49\textwidth]{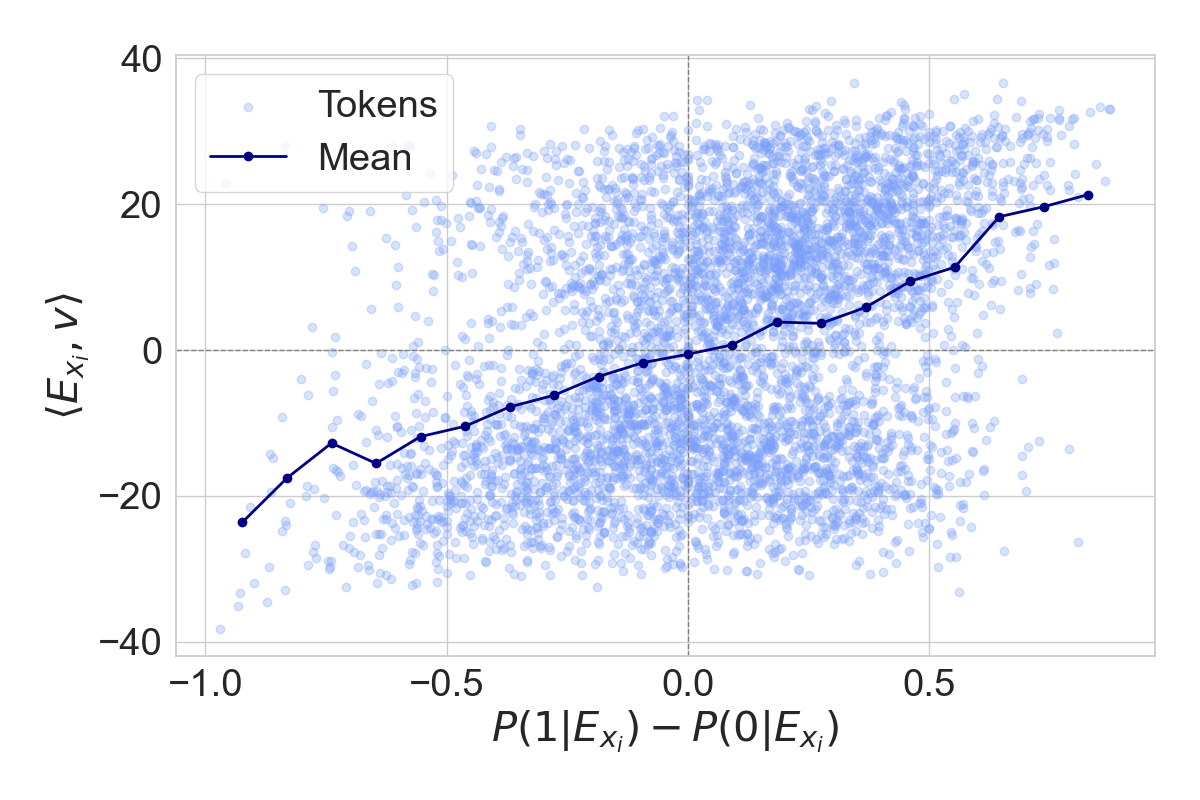}
  \caption{Dot-product of token embeddings with $\cls$ embedding $p$ and regression coefficients $v$, as a function of the token-wise difference in posterior for IMDB dataset (top row) and Yelp dataset (bottom row). We consider the one-layer attention model in \eqref{eq:f} with all parameters trained until convergence.}
  \label{fig:imdb_corr}
\end{figure}

\textbf{IMDB and Yelp datasets.} The IMDB dataset\footnote{\url{https://www.kaggle.com/datasets/lakshmi25npathi/imdb-dataset-of-50k-movie-reviews}} consists of $50000$ reviews of average length $239$ words per review, associated to either a positive or a negative sentiment.
Yelp reviews\footnote{\url{https://www.kaggle.com/datasets/yelp-dataset/yelp-dataset}} provide a much larger selection. To align the data size and sequence length with the IMDB dataset, we randomly subsample a portion of the Yelp dataset constrained on the sequence length, i.e., we select reviews which have at least $1000$ and not more than $1500$ characters. In addition, Yelp reviews provide a five-star ranking, which we convert to the binary sentiment based on the following rule: 1/2 stars reviews are assigned label $-1$; 4/5 star reviews are assigned label $+1$; neutral reviews, i.e.,  3-star score, are removed. We adhere to a typical preprocessing pipeline for both datasets: we start by cleaning the data from punctuation symbols and omitting the stop-words, followed by an application of stemming; and we use the Bert tokenizer from Hugging Face\footnote{\url{https://huggingface.co/google-bert/bert-base-uncased}} to tokenize sequences. Tokens that appear less than $50$ times are purged.

The numerical simulations for both datasets are reported in Figure \ref{fig:imdb_corr}, which displays a phenomenology similar to that obtained for synthetic data in Figure \ref{fig:syn_corr}, thus providing additional grounding for our theoretical claims.

\textbf{Two-layer model.} We also consider the following two-layer model: 
\begin{equation}\label{eq:2layer_def}
\begin{split}
   \mE'_X &= \mathrm{LayerNorm}(\sfmax(\mE_X \mE_X^\top) \mE_X+\mE_X), \hspace{2mm}     f(X;p,\mE) = \sfmax( p^\top (\mE'_X)^\top ) \mE'_X v,
\end{split}
\end{equation}
which includes both a skip connection and the layer-norm.
We note that, for both IMDB and Yelp data, the model in  \eqref{eq:2layer_def} achieves significantly smaller loss values at convergence (of the order of $10^{-5}$, in contrast to the order of $10^{-1}$ achieved by the model in \eqref{eq:f}). However, even if this model is more complex than the one analyzed in Section \ref{sec:1step_1impt}, the results in Figure \ref{fig:imdb_corr_2layer} are still remarkably similar to those in Figures \ref{fig:syn_corr} and \ref{fig:imdb_corr}. 

Finally, we note that all plots consider on the $x$-axis the difference in posterior probabilities
\begin{equation}\label{eq:p_diff}
    p(1|\mE_{x_i}) - p(0|\mE_{x_i}) = \frac{\sum_{(X, y)\in\mathcal D}y \sum_{i=1}^T\idc_{x_i =s}}{\sum_{(X, y)\in\mathcal D}\sum_{i=1}^T\idc_{x_i =s}} 
\end{equation}
in place of the quantity $\alpha_s$ defined in  \eqref{eq:alphas}. In fact, while the quantity in \eqref{eq:alphas} appears naturally from the analysis of gradient descent, the difference in posterior probabilities provides better visuals for real data (IMDB and Yelp). The difference between \eqref{eq:alphas} and \eqref{eq:p_diff}  lies in the normalization used: the posterior difference in \eqref{eq:p_diff} is the discrepancy between counts of the token $x_i$ in positive and negative sentences normalized by the total number of occurrences of $x_i$, while the quantity in \eqref{eq:p_diff}  normalizes the discrepancy by the total number of tokens $nT$ in the datasets. For synthetic data sampled according to \eqref{eq:synth_data_exp}, due to the uniform nature of the sampling procedure, all tokens appear the same number of times. Thus, both quantities are the same up to a fixed scaling and, thus, they are equivalent.
Additional details on the hyperparameter settings for all the experimental setups are contained in Appendix \ref{appendix:hyperparams}.

\section{Conclusions and limitations}
\label{sec:discuss_lim}

In this paper, we study how the embedding vectors trained via gradient methods capture the importance of different tokens in the dataset. We theoretically characterize \emph{(i)} the context embedding $E_s$ after one gradient step, and \emph{(ii)} the implicit bias of the $\cls$ embedding $p$ after training with gradient flow until convergence. We conduct experiments on synthetic and realistic datasets which demonstrate the generality of our findings. 

A limitation of our work is that the characterization we put forward is only in terms of the first-order statistics of the tokens (i.e., the frequencies with which they occur in the dataset), and it does not describe how the model learns the causal structure between tokens. In practice, both first-order statistics and causal structure are expected to be crucial for the model to ``understand'' a text. While our theory assumes a one-layer attention model, the numerical results of Figure \ref{fig:imdb_corr_2layer} suggests that a similar qualitative picture holds more generally. 
This prompts us to conjecture that in deeper attention models with multiple heads, the earlier layers form induction heads \citep{olsson2022context} which learn the causal structure between tokens, and later layers perform classification based on the empirical statistics of the resulting $k$-tuples. We regard this investigation as an exciting future direction.

\section*{Acknowledgements}

This research was funded in whole or in part by the Austrian Science Fund (FWF) 10.55776/COE12. For the purpose of open access, the authors have applied a CC BY public copyright license to any Author Accepted Manuscript version arising from this submission. MM and AS were partially funded by the European Union (ERC, INF2, project number 101161364). Views and opinions expressed are however those of the author(s) only and do not necessarily reflect those of the European Union or the European Research Council Executive Agency. Neither the European Union nor the granting authority can be held responsible for them.
AS was also supported by the SNF Grant 204439. SO was supported by the Office of Naval Research grant N000142412289.

\bibliography{references}

\appendix

\newpage

\paragraph{Additional notation.} Throughout the appendices, to simplify the notation, we write 
\begin{equation}
 a_i(X) := p^\top E_{x_i},\qquad    q_i(X) := \frac{\exp(a_i(X)) }{ \sum_{j=1}^T \exp(a_{j}(X))},
\end{equation}
so that $f(X; p, \mE) = \sum_{i=1}^T q_i(X)E_{x_i}^\top v$. We will drop the dependence on $X$ in $a_i(X), q_i(X)$ when there is no confusion.
We also denote 
\begin{equation}
\gamma_i(X,y) := y E_{x_i}^\top v,    
\end{equation}
dropping again the dependency on $X,y$ when there is no confusion. Finally, we define 
\begin{equation}    
g(X,y):= \frac{1}{1+\exp(y f(X; p, \mE))}.
\end{equation}

\paragraph{Properties of initialization.}

By standard concentration inequalities, with probability at least $ 1-\delta$, at initialization we have \begin{equation}\label{eq:cond}
\begin{split}
        \max &\left\{\max_{s\neq s' \in \cS}|E_s^\top E_{s'} |,  
        \max_{s \in \cS} |E_s^\top v |, \max_{s \in \cS} |E_s^\top p|, |p^\top v| 
        \right\} \leq \frac{1}{\sqrt{d}} \sqrt{2 \log\frac{ |\cS|^2}{\delta}},\\
        \max &\left\{\max_{s\in \mathcal S}\|E_s\|_2,\|p\|_2\right\}\le 2, \qquad \min_{s \in \cS} \|E_s\|_2 \geq \frac{1}{2}.
\end{split}
\end{equation}
For all results of the paper holding with probability at least $1-\delta$, we will be implicitly conditioning on \eqref{eq:cond}.

\section{Technical lemmas}

\begin{lemma}\label{lemma:gdupdate}
The gradients of the empirical loss are given by \begin{equation*}
    \begin{split}
        \nabla_{E_s} \cL(\mE,p)  &= - \hbE\left[y g(X,y) \left( \sum_{i=1}^T  (\sum_{j\neq i} ( \idc_{x_i=s} - \idc_{x_j=s}) q_{i}(X) q_{j}(X) ) E_{x_i}^\top v p + \sum_{i=1}^T \idc_{x_i = s} q_{i} v \right) \right], \\
        \nabla_{p} \cL(\mE,p)  &= - \hbE\left[y g(X,y) \left( \sum_{i=1}^T (\sum_{j\neq i} q_{i}(X) q_{j}(X) (E_{x_i} - E_{x_j})) E_{x_i}^\top v\right)  \right],
    \end{split}
\end{equation*} where we have defined $g(X,y) = \frac{1}{1+\exp(y f(X))}$.
\end{lemma}

\begin{proof}
 We start by taking the gradient of $q_{i}$ as \begin{equation*}
 \begin{split}
     \nabla_{E_s} q_{i}(X) &= \frac{\idc_{x_i=s} \exp\left(E_{x_i}^\top p \right) p \left(\sum_{j=1}^T  \exp(E_{x_j}^\top p)\right) - \left(\sum_{j=1}^T \idc_{x_j=s}  \exp(E_{x_j}^\top p) p\right) \exp(E_{x_i}^\top p ) }{\left(\sum_{j=1}^T  \exp(E_{x_j}^\top p)\right)^2} \\
     & = \frac{ p \sum_{j=1}^T( \idc_{x_i=s} - \idc_{x_j=s}) \exp(E_{x_j}^\top p)\exp\left(E_{x_i}^\top p\right)}{\left(\sum_{j=1}^T  \exp(E_{x_j}^\top p)\right)^2} \\
     & = p \left(\sum_{j=1}^T ( \idc_{x_i=s} - \idc_{x_j=s}) q_{i} q_{j} \right) \\
     & = p \left(\sum_{j\neq i} ( \idc_{x_i=s} - \idc_{x_j=s}) q_{i} q_{j} \right), \\
     \end{split}
     \end{equation*}
     \begin{equation*}
 \begin{split}
     \nabla_p q_{i}(X) &= \frac{\left(\exp(E_{x_i}^\top p) E_{x_i}\right) \left(\sum_{j=1}^T \exp(E_{x_j}^\top p)\right) - \sum_{j=1}^T \exp(E_{x_j}^\top p) E_{x_j} \exp(E_{x_i}^\top p) }{\left(\sum_{j=1}^T \exp(E_{x_j}^\top p)\right)^2} \\
     & = \frac{\sum_{j=1}^T \exp(E_{x_j}^\top p) \exp(E_{x_i}^\top p) (E_{x_i} - E_{x_j} )}{\left(\sum_{j=1}^T \exp(E_{x_j}^\top p)\right)^2} \\
     & = \sum_{j=1}^T q_{i} q_{j} (E_{x_i} - E_{x_j}) \\
     & = \sum_{j\neq i} q_{i} q_{j} (E_{x_i} - E_{x_j}) .
     \end{split}
 \end{equation*}
Next, we look at the gradient of $f(X; p, \mE)$:\begin{equation*}
    \begin{split}
        \nabla_{E_s} f(X; p, \mE) &= \sum_{i=1}^T \left(\nabla_{E_s} q_{i} \right)E_{x_i}^\top v + \sum_{i=1}^T \idc_{x_i = s}  q_{i} v  \\
        & =\sum_{i=1}^T  \left(\sum_{j\neq i} ( \idc_{x_i=s} - \idc_{x_j=s}) q_{i} q_{j} \right) E_{x_i}^\top v p + \sum_{i=1}^T \idc_{x_i = s}  q_{i} v, \\
        \nabla_p f(X; p, \mE) & = \sum_{i=1}^T \left(\sum_{j\neq i} q_{i} q_{j} (E_{x_i} - E_{x_j})\right) E_{x_i}^\top v.
    \end{split}
\end{equation*} 
This allows us to conclude that 
\begin{equation*}
    \begin{split}
        \nabla_{E_s} \cL(\mE,p) &= \hbE\left[\frac{-y}{1+\exp(y f(X; p, \mE))} \nabla_{E_s} f(X; p, \mE)  \right] \\
        & =\hbE\Biggl[\frac{-y}{1+\exp(y f(X; p, \mE))} \biggl( \sum_{i=1}^T  \left(\sum_{j\neq i} ( \idc_{x_i=s} - \idc_{x_j=s}) q_{i}(X) q_{j}(X) \right) E_{x_i}^\top v p \\
        &\hspace{23em}+ \sum_{i=1}^T \idc_{x_i = s} q_{i}(X) v \biggr) \Biggr] , \\
        \nabla_{p} \cL(\mE,p) &= \hbE\left[\frac{-y}{1+\exp(y f(X; p, \mE))} \nabla_{p} f(X; p, \mE)  \right] \\
        & = \hbE\left[\frac{-y}{1+\exp(y f(X; p, \mE))} \left( \sum_{i=1}^T \left(\sum_{j\neq i} q_{i}(X) q_{j}(X) (E_{x_i} - E_{x_j})\right) E_{x_i}^\top v\right)  \right],
    \end{split}
\end{equation*}
thus concluding the proof.
\end{proof}

\begin{lemma}
\label{lem:dir_grad}
    For any vector $\widehat{p},$ we have \begin{equation*}
         - \widehat{p}^\top \nabla_p\mathcal L(\mE,p) = \hbE \left[ g(X,y) \left(\sum_{i=1}^T\sum_{j>i} (\widehat{a}_i(X) - \widehat{a}_j(X)) q_i(X) q_j(X) (\gamma_i(X,y) - \gamma_j(X,y)) \right) \right],
    \end{equation*} where $\widehat{a}_i=\widehat{p}^\top E_{x_i}$ for all $i\in\{1, \ldots, T\}$. 
\end{lemma}

\begin{proof}
    From Lemma \ref{lemma:gdupdate}, we have \begin{equation*}
        \begin{split}
            \nabla_{p} \cL(\mE,p)  &= - \hbE\left[y g(X,y) \left( \sum_{i=1}^T \left(\sum_{j\neq i} q_{i}(X) q_{j}(X) (E_{x_i} - E_{x_j})\right) E_{x_i}^\top v\right)  \right] \\
            &= - \hbE\left[ g(X,y) \left( \sum_{i=1}^T \left(\sum_{j\neq i} q_{i}(X) q_{j}(X) (E_{x_i} - E_{x_j})\right) \gamma_i(X,y) \right)  \right]\\
            & = - \hbE\left[ g(X,y) \mE_X^\top \left(\Diag(q_X) - q_X q_X^\top \right) \gamma(X,y)  \right],
        \end{split} 
    \end{equation*} where $q_X = [q_1(X), \dots, q_T(X)]^\top, \gamma(X,y) = [\gamma_1(X,y), \dots \gamma_T(X,y)]^\top$ and $\Diag(q_X)$ denotes the diagonal matrix with $[\Diag(q_X)]_{i,i} = q_i(X).$

    Thus, letting $\widehat{a}=[\widehat{a}_1, \ldots, \widehat{a}_T]\in \mathbb R^T$ with $\widehat{a}_i = \hat{p}^\top E_{x_i}$, we have \begin{equation*}
        \begin{split}
            -\widehat{p}^\top \nabla_p\cL(\mE,p) & =  \hbE\left[ g(X,y) \widehat{p}^\top \mE_X^\top (\Diag(q_X) - q_X q_X^\top ) \gamma(X,y)  \right] \\
            & = \hbE\left[ g(X,y) \widehat{a}^\top (\Diag(q_X) - q_X q_X^\top ) \gamma(X,y)  \right] \\
            & = \hbE\left[ g(X,y) \left( \sum_{i=1}^T  \widehat{a}_i q_i (1-q_i) \gamma_i - \sum_{i=1}^T \sum_{j \neq i} \widehat{a}_i q_i q_j \gamma_j\right) \right] \\
            & = \hbE\left[ g(X,y) \left( \sum_{i=1}^T \sum_{j \neq i} \widehat{a}_i q_i q_j (\gamma_i - \gamma_j)\right) \right] \quad (\text{use $1-q_i = \sum_{j \neq i} q_j$}) \\
            & = \hbE\left[ g(X,y) \left( \frac{1}{2}  \sum_{i=1}^T \sum_{j \neq i} \widehat{a}_i q_i q_j (\gamma_i - \gamma_j) + \frac{1}{2} \sum_{j=1}^T \sum_{i \neq j} \widehat{a}_j q_i q_j (\gamma_j - \gamma_i) \right) \right] \\
            & = \hbE\left[ g(X,y) \left( \frac{1}{2}  \sum_{i=1}^T \sum_{j \neq i} (\widehat{a}_i - \widehat{a}_j )q_i q_j (\gamma_i - \gamma_j) \right) \right] \\
            & = \hbE\left[ g(X,y) \left( \sum_{i=1}^T \sum_{j > i} (\widehat{a}_i - \widehat{a}_j )q_i q_j (\gamma_i - \gamma_j) \right) \right].
        \end{split}
    \end{equation*}
\end{proof}

\begin{lemma}[Convergence lemma]
\label{lem:converge}
    Let $\|p_t\|_2 \rightarrow \infty$ and suppose there exists  $\widehat{p}$ such that, for any $\eps >0$, there is a $\bar t(\eps)$ ensuring \begin{equation}
    \label{eq:dir_grad}
        -\frac{\widehat{p}^\top}{\|\widehat{p}\|_2} \nabla_p \cL(\mE, p_t) \geq -(1-\eps) \frac{p_t^\top}{\|p_t\|_2} \nabla_p \cL(\mE, p_t),\qquad \text{ for all }t\ge \bar t(\epsilon).
     \end{equation} Then, if $\lim_{t \rightarrow \infty} \frac{p_t}{\|p_t\|_2}$ exists, we have \begin{equation*}
         \lim_{t \rightarrow \infty} \frac{p_t}{\|p_t\|_2} = \frac{\widehat{p}}{\|\widehat{p}\|_2}.
     \end{equation*}
\end{lemma}

\begin{proof}
    By the definition of the gradient flow, \eqref{eq:dir_grad} is equivalent to \begin{equation*}
        \frac{\widehat{p}^\top}{\|\widehat{p}\|_2} \frac{\dd p_t}{\dd t} \geq (1-\eps) \frac{p_t^\top}{\|p_t\|_2} \frac{\dd p_t}{\dd t}.
    \end{equation*}

    We note that \begin{equation*}
        \frac{p_t^\top}{\|p_t\|_2} \frac{\dd p_t}{\dd t} = \frac{\dd}{\dd t} \|p_t\|_2.
    \end{equation*}

    Thus, by integrating both sides from $[\bar t(\epsilon), t],$ we have: \begin{equation*}
       \frac{\widehat{p}^\top}{\|\widehat{p}\|_2} (p_t - p_{\bar t(\eps)}) \geq (1-\eps) (\|p_t\|_2 - \|p_{\bar t(\eps)}\|_2),
    \end{equation*}
which gives \begin{equation*}
        \frac{\widehat{p}^\top p_t}{\|\widehat{p}\|_2 \|p_t\|_2} \geq (1-\eps) - (1-\eps) \frac{ \|p_{\bar t(\eps)}\|_2 }{\|p_t\|_2} + \frac{\widehat{p}^\top p_{\bar t(\eps)}}{\|\widehat{p}\|_2 \|p_t\|_2}.
    \end{equation*} Since $p_{\bar t(\eps)}, \widehat{p}$ have finite norm for fixed $\eps$, by taking the limit on both sides, we have \begin{equation*}
        \liminf_{t \rightarrow \infty}   \frac{\widehat{p}^\top p_t}{\|\widehat{p}\|_2 \|p_t\|_2} \geq 1-\eps.
    \end{equation*}

    As we assume that $\lim_{t \rightarrow \infty} \frac{p_t}{\|p_t\|_2}$ exist and the above argument holds for any $\eps$, we conclude \begin{equation*}
        \lim_{t \rightarrow \infty} \frac{p_t}{\|p_t\|_2} = \frac{\hat{p}}{\|\hat{p}\|_2}.
    \end{equation*}
\end{proof}

\begin{lemma}
\label{lem:bds_on_q}
    Given a sequence $X$, model parameters $\mE,p,v$, and indices $i_*, j$ s.t.\ $x_{i_*}\in \cS_X(p), x_j \in X \setminus \cS_X(p),$ the following results hold. \begin{enumerate}
        \item We have \begin{equation*}
            \frac{1}{T} \leq q_{i_*} \leq 1.
        \end{equation*}
        \item If there exist $\tau > 0$ such that $p^\top (E_{x_{i_*}} - E_{x_{j}}) \geq  \tau$ for all $x_{i_*} \in \cS_X(p)$, then we have \begin{equation*}
            q_j \leq \frac{1}{1 + \exp(\tau)}.
        \end{equation*}
        \item If there exist $\tau > 0$ such that $p^\top (E_{x_{i_*}} - E_{x_{j}}) \leq  \tau$ for all $x_{i_*} \in \cS_X(p)$, then we have \begin{equation*}
        \begin{split}
             q_j \geq \frac{1}{T \exp(\tau )}.
        \end{split}
        \end{equation*}
    \end{enumerate}
\end{lemma}

\begin{proof}
    The upper bound on $q_{i_*}$ is trivial. For the lower bound: \begin{equation*}
        \begin{split}
            q_{i_*} &= \frac{\exp(p^\top E_{x_{i_*}})}{\exp(p^\top E_{x_{i_*}}) + \sum_{j \neq i_*} \exp(p^\top E_{x_j})} \\
            & \geq \frac{\exp(p^\top E_{x_{i_*}})}{T \exp(p^\top E_{x_{i_*}}) } = \frac{1}{T}.
        \end{split}
    \end{equation*}

    If there exists $\tau > 0$ such that $p^\top (E_{x_{i_*}} - E_{x_{j}}) \geq  \tau$ for all $x_i \in \cS_X(p)$, then we have \begin{equation*}
        \begin{split}
            q_j &= \frac{1}{1 + \sum_{i \neq j} \exp(p^\top(E_{x_{i}} - E_{x_{j}}))} \\
            & \leq \frac{1}{1 + \exp(p^\top(E_{x_{i_*}} - E_{x_{j}} ))}  \\
            & \leq \frac{1}{1 + \exp(\tau)}.
        \end{split}
    \end{equation*}

    If there  exists $\tau > 0$ such that $p^\top (E_{x_{i_*}} - E_{x_{j}}) \leq  \tau$ for all $x_{i_*} \in \cS_X(p)$, then  we have \begin{equation*}
        \begin{split}
            q_{i_*} &= \frac{1}{1 + \sum_{i \neq j} \exp(p^\top(E_{x_{i}} - E_{x_{j}}))} \\
            &\geq \frac{1}{1 + (T -1 )\exp(p^\top(E_{x_{i_*}} - E_{x_{j}}) ) }  \quad (\text{by definition of $\cS_X(p)$}) \\
            &\geq \frac{1}{T \exp(\tau)}.
        \end{split}
    \end{equation*}
\end{proof}

\section{Properties after the first gradient step}


\begin{lemma}[Boundedness of the embeddings]
\label{lem:bd_embd}
    For any $\delta > 0,$ let \begin{equation*}
        d \geq \max\left\{ 256,\left( 2 \log\frac{|\cS|^2}{\delta} \right)^2   \right\},
    \end{equation*} then with probability at least $1-\delta,$ \begin{equation*}
        \max_{s\in \mathcal S} \|E_s^1\|_2 \leq 2(1 +  2\eta_0 ), \qquad \|p^1\|_2 \leq 2+11\eta_0 d^{-\frac{1}{4}}.
    \end{equation*}
\end{lemma}

\begin{proof}
By using \eqref{eq:cond}, we have that \begin{equation*}
        \begin{split}
            \max_{s\in \mathcal S}\|E_s^1\|_2 &\leq  \max_s \left( \|E_s^0\|_2 + \frac{\eta_0}{2} \|v\|_2 + \|err_s\|_2 \right) \\
            & \leq \max_{s\in \mathcal S}  \left( 2 + \frac{\eta_0}{2} + 11\eta_0 d^{-\frac{1}{4}}   \right) 
            \\
            &\leq  2 + 4\eta_0 , 
        \end{split}
    \end{equation*}
    and that
    \begin{equation}
        \|p^1\|_2\le \|p^0\|_2+\|\err_p\|_2\le 2+11\eta_0 d^{-\frac{1}{4}}.
    \end{equation}
\end{proof}

\begin{lemma}[Upper bound on the loss]
\label{lem:un_loss_1step}
 For any $\delta > 0,$ let \begin{equation*}
       d \geq \max\left\{256,  \left(2 \log\frac{|\cS|^2}{\delta} \right)^2, (88 \eta_0^2+111\eta_0+2)^8  \right\},
    \end{equation*} then with probability at least $1-\delta$, \begin{equation*}
    \begin{split}
        \cL(\mE^1, p^1)  \leq \hbE\left[\log (1+\exp(- \frac{1}{T} \sum_{i=1}^T \frac{\eta_0}{2} y\alpha_{x_i} + \frac{1}{22 \eta_0} )) \right].
    \end{split}
    \end{equation*}
\end{lemma}

\begin{proof}
    We first lower bound  $y f(X; p, \mE)$ for each pair $X,y.$ After the first step, we have \begin{equation*}
    \begin{split}
        \max_{s,s'} |(p^1)^\top (E^1_s - E^1_{s'})| &= \max_{s,s'} \big| (p^0)^\top (E_s^0 - E_{s'}^0) + \frac{\eta_0}{2} (\alpha_s - \alpha_{s'}) (p^0)^\top v \\
        &\hspace{4em}+ \err_p^\top (E_s^1 - E^1_{s'}) + (\err_s - \err_{s'})^\top p^1  \big|.
    \end{split}
    \end{equation*} We bound each term separately: \begin{equation*}
        \begin{split}
            \max_{s,s'} |(p^0)^\top (E_s^0 - E_{s'}^0)| &\leq 2 \max_{s} |(p^0)^\top E_s^0| \leq 2 d^{- \frac{1}{4}}, \\
             \frac{\eta_0}{2} (\alpha_s - \alpha_{s'}) |(p^0)^\top v| &\leq \eta_0 |(p^0)^\top v| \leq \eta_0 d^{- \frac{1}{4}}, \\
             |\err_p^\top (E_s^1 - E^1_{s'} )|&\leq \|\err_p^\top\|_2 \|E_s^1 - E^1_{s'}\|_2 \leq 44\eta_0 d^{-\frac{1}{4}} (1+2\eta_0), \\
             |(\err_s - \err_{s'})^\top p^1| &\leq 2 \|p^1\|_2 \max_{s} \|\err_s\|_2 \leq 22 \eta_0 d^{- \frac{1}{4}}\left(2+11\eta_0 d^{-\frac{1}{4}}\right),
        \end{split}
    \end{equation*}
where we have used \eqref{eq:cond}.
    By picking $d \geq (88 \eta_0^2+111\eta_0+2)^8,$ we get $\max_{s,s'}  |(p^1)^\top (E^1_s - E^1_{s'})| \leq d^{-\frac{1}{8}},$ which implies that, for any $X$ and any $i \in \{1, \ldots, T\},$  \begin{equation*}
        \begin{split}
            \frac{1}{T} - \frac{2d^{-\frac{1}{8}}}{T} \leq q_i(X) \leq \frac{1}{T} + \frac{2d^{-\frac{1}{8}}}{T}.
        \end{split}
    \end{equation*}
    Thus, we lower bound $yf(X; p, \mE)$ for each pair $(X,y)$ as \begin{equation*}
        \begin{split}
            yf(X; p, \mE) &= \sum_{i=1}^T q_{i}(X) \gamma_i(X) \\
            & \geq \frac{1}{T} \sum_{i=1}^T \frac{\eta_0}{2} y \alpha_{x_i} -  \sum_{i=1}^T \frac{2d^{-\frac{1}{8}}}{T} \frac{\eta_0}{2} \alpha_{x_i} + \sum_{i=1}^T y q_i(X) v^\top (E^0_{x_i} + \err_{x_i}) \\
            & \geq \frac{1}{T} \sum_{i=1}^T  \frac{\eta_0}{2} y\alpha_{x_i} -  d^{-\frac{1}{8}} \eta_0 - (1 + 2d^{-\frac{1}{8}})  v^\top(E^0_{x_i} + \err_{x_i}) \\
            & \geq \frac{1}{T} \sum_{i=1}^T \frac{\eta_0}{2} y\alpha_{x_i} -  d^{-\frac{1}{8}} \eta_0 - 3(1+11\eta_0) d^{- \frac{1}{4}} \\
            & \geq \frac{1}{T} \sum_{i=1}^T \frac{\eta_0}{2} y\alpha_{x_i} - \frac{1}{22 \eta_0},
        \end{split}
    \end{equation*}
    which allows us to conclude that \begin{equation}
    \begin{split}
        \cL(\mE^1, p^1) &= \hbE\left[\log (1+\exp(-yf(X; p, \mE))) \right] \\
        & \leq \hbE\left[\log (1+\exp(- \frac{1}{T} \sum_{i=1}^T \frac{\eta_0}{2} y\alpha_{x_i} + \frac{1}{22 \eta_0} )) \right].
    \end{split}
    \end{equation}
\end{proof}

\section{Proofs for Section \ref{sec:1step_1impt}}

\subsection{Proof of Lemma \ref{lem:overlap_v}}
\label{apx:pf_lem_overlap_v}

For simplicity, in the proof we drop the time dependency in all the variables. 
By picking \begin{equation*}
    d \geq \left(2 \log\frac{|\cS|^2}{\delta} \right)^2,
\end{equation*} from \eqref{eq:cond} we have \begin{equation*}
\begin{split}    
     \max \left\{
     \max_{s \in \cS} |E_s^\top v |, \max_{s \in \cS} |E_s^\top p|, |p^\top v| 
     \right\} &\leq d^{-\frac{1}{4}},
     \\
    \max \left\{\max_{s\in \mathcal S}\|E_s\|_2,\|p\|_2\right\}&\le 2.
\end{split}
\end{equation*}
Thus, at initialization, we have that, for all $s$, \begin{equation*}
    \exp\left( - d^{-\frac{1}{4}} \right) \leq \exp(p^\top E_s) \leq \exp\left( d^{-\frac{1}{4}} \right),
\end{equation*} which implies that, for any sequence $X$ and any position $i$, \begin{equation*}
    \frac{1}{T +2T\left( d^{-\frac{1}{4}} \right)} \leq \frac{1}{1 + (T-1) \exp\left( 2d^{-\frac{1}{4}} \right)} \leq q_i(X) \leq \frac{1}{1 + (T-1) \exp\left( -2d^{-\frac{1}{4}} \right)} \leq \frac{1}{T - 2T\left(d^{-\frac{1}{4}} \right)},
\end{equation*} where we use the fact that for $z \in [-1,1], 1 - |z| \leq \exp(z) \leq 1+|z|.$

Furthermore, for  $d>256$ and for any sequence $(X,y),$ we have \begin{equation*}
    \frac{1}{T} - \frac{4 d^{-\frac{1}{4}}}{T} \leq q_{i}(X) \leq \frac{1}{T} + \frac{4 d^{-\frac{1}{4}}}{T},
\end{equation*} and 
\begin{equation*}
    - 2 d^{-\frac{1}{4}}\leq \frac{-T d^{-\frac{1}{4}}}{T - 2Td^{-\frac{1}{4}} } \leq y f(X; p, \mE) \leq \frac{T d^{-\frac{1}{4}}}{T - 2Td^{-\frac{1}{4}} } \leq 2 d^{-\frac{1}{4}}.
\end{equation*}  Then,\begin{equation*}
    g(X,y) \leq \frac{1}{1+\exp(-2 d^{-\frac{1}{4}})} \leq \frac{1}{2 - 2 d^{-\frac{1}{4}}} \leq \frac{1}{2} + d^{-\frac{1}{4}},
\end{equation*} and similarly \begin{equation*}
    g(X,y) \geq \frac{1}{2} -d^{-\frac{1}{4}}.
\end{equation*}

Now we look at the gradient update of the first step. By Lemma \ref{lemma:gdupdate}, we have \begin{equation*}
\begin{split}
    - \nabla_{E_s} \cL(\mE,p) &= \hbE\left[ y g(X,y) \left( \sum_{i=1}^T\left( \sum_{j \neq i } (\idc_{x_i=s} - \idc_{x_j=s}) q_i q_j \right) E_{x_i}^\top v p + \sum_{i=1}^T \idc_{x_i=s} q_i v \right) \right] \\
    & = \frac{1}{2 T} \hbE\left[ y \sum_{i=1}^T \idc_{x_i = s} \right] v \\
    & \quad + \frac{1}{2} \hbE\left[ y \sum_{i=1}^T \idc_{x_i = s} \left(q_i-\frac{1}{T}   \right) \right] v \\
    & \quad + \hbE\left[  y g(X,y) \left( \sum_{i=1}^T\left( \sum_{j \neq i } (\idc_{x_i=s} - \idc_{x_j=s}) q_i q_j \right) E_{x_i}^\top v p \right)\right]\\
    &\quad+\hbE\left[ y \left(g(X,y) - \frac{1}{2}\right)\sum_{i=1}^T \idc_{x_i=s} q_i v \right], \\
    - \nabla_{p} \cL(\mE, p) &= \hbE \left[ y g(X,y) \left( \sum_{i=1}^T \left(\sum_{j \neq i} q_i q_j (E_{x_i} - E_{x_j}) E_{x_i}^\top v\right)\right) \right] .
\end{split}
\end{equation*}

We note that \begin{equation*}
    \frac{1}{2 T} \hbE\left[ y \sum_{i=1}^T \idc_{x_i = s} \right] v = \frac{1}{2} \alpha_s v,
\end{equation*} and we bound the remaining error terms.

We have that \begin{equation*}
    \begin{split}
        \left\|\frac{1}{2} \hbE\left[ y \sum_{i=1}^T \idc_{x_i = s} \left(q_i-\frac{1}{T} \right) \right] v \right\|_2 \leq  d^{-\frac{1}{4}},
    \end{split}
\end{equation*} and \begin{equation*}
    \begin{split}
        \Biggl\|\hbE\Biggl[  y g(X,y) &\left( \sum_{i=1}^T\left( \sum_{j \neq i } (\idc_{x_i=s} - \idc_{x_j=s}) q_i q_j \right) E_{x_i}^\top v p \right)\\
        &\hspace{12em} + y \left(g(X,y) - \frac{1}{2}\right)\sum_{i=1}^T \idc_{x_i=s} q_i v \Biggr]\Biggr\|_2 \leq  10 d^{-\frac{1}{4}} .
    \end{split}
\end{equation*}
Furthermore, we also have that \begin{equation*}
    \|\nabla_{p} \cL(\mE, p) \|_2 \leq 8 d^{-\frac{1}{4}}.
\end{equation*}
Thus, the desired claim follows. 

\subsection{Proof of Lemma \ref{lem:sol_mm}}
\label{apx:pf_lem_sol_mm}

\begin{proof}
We first show that, if \eqref{eqn:max-margin} is feasible, then the solution is unique. Indeed, assume by contradiction that $p_1, p_2$ are two different solutions of \eqref{eqn:max-margin}. Clearly, $p_1$ and $p_2$ have the same norm, so $\frac{p_1^\top p_2}{\|p_1\|_2 \|p_2\|_2} \neq 1.$ Then, any convex combination of $p_1, p_2$ gives a feasible solution with a strictly smaller norm, which is a contradiction.

Next, we show that \eqref{eqn:max-margin} is always feasible. To see this, by definition, there exists some $\tau$ such that \begin{equation*}
        p_\circ^\top( E_s - E_{s'}) \geq \tau,\qquad  \forall s \in \cS_X(p_\circ),\,\, \forall s' \in \overline{\cS_X(p_\circ)},\,\, \forall X \in \cX_n.
    \end{equation*} Then, $\frac{p_\circ}{\tau}$ is a feasible solution of \eqref{eqn:max-margin} which concludes the proof. 
\end{proof}

\subsection{Proof of Theorem \ref{thm:complete_seq}}
\label{apx:pf_complete_seq}
We prove each part separately. We first show that $\lim_{t\to\infty}\|p_t\|_2 =\infty.$ \begin{lemma}
    \label{lem:1_impt_norm}
    Under Assumption \ref{asm:complete_token}, for any $\delta>0,$ by picking \begin{equation*}
       d \geq \max\left\{ 256, \left(2 \log\frac{|\cS|^2}{\delta} \right)^2, |\cS| +3\right\},
    \end{equation*} with probability at least $1-\delta$, we have  $\lim_{t\to\infty}\|p_t\|_2 =\infty$. 
\end{lemma}

\begin{proof}
    It is sufficient to show that there exists a non-zero finite-norm $\widehat{p},$ such that for any finite norm $p,$ \begin{equation*}
        \widehat{p}^\top \nabla_p\cL(\mE^1, p) \neq 0.
    \end{equation*}

    Indeed, the above condition means that there is no stationary point for any finite-norm $p.$ For gradient flow, we have that \begin{equation*}
        \lim_{t \rightarrow \infty} \nabla_p \cL(\mE^1, p_t) =0,
    \end{equation*} 
    which by contradiction implies the desired result.
    
    Now we construct such $\widehat{p}.$ Since $d > |\cS| +2,$ we have that with high-probability $\mE^0$ is full row rank.
    Furthermore, $\mE^1 = \mE^0 + \mathbf{\Delta}$ and each row of $\mathbf{\Delta}$ is in the subspace generated by $v$ and  $p_0$. Thus, we can pick $\widehat{p} \perp v, p_0,$ so that \begin{equation*}
        \mE^1 \widehat{p} = \mE^0 \widehat{p}.
    \end{equation*} 
    Furthermore, for any fixed $v,p_0,$ let $u_1, u_2$ be the orthogonormal basis of $span(v,p_0)$. Then, with probability $1$ we also have $ \mE^0 \left(I - u_1 u_1^\top - u_2 u_2^\top \right)$ also has full row rank. 
    
    Without loss of generality, let $x_1$ be an important token in a positive sequence $X_k$, i.e., $\gamma_1(X_k) \geq \frac{\eta_0}{4 n T}$.
    Then, we define $a \in \R^{|\cS|}$ such that $a_{1} = 1$ and $a_i =0$ for all $i \neq 1$. Let \begin{equation*}
         \mE^0 \left(I - u_1 u_1^\top - u_2 u_2^\top \right) \bar{p} = a.
    \end{equation*} As $ \mE^0 \left(I - u_1 u_1^\top - u_2 u_2^\top \right)$ has full row rank, 
    there exists a non-zero $\bar{p}$ that solves the above equation. By letting $\hat{p} = \left(I - u_1 u_1^\top - u_2 u_2^\top \right) \bar{p},$ we know that $\hat{p} \perp v,p_0,$ and $\mE^0 \hat{p}= a. $
    By Lemma \ref{lem:dir_grad}, we have that, for any $p,$ \begin{equation*} 
    \begin{split}
        -\widehat{p}^\top \nabla_p\cL(\mE^1, p) &= \hbE \left[ g(X,y) \left(\sum_{i=1}^T\sum_{j>i} (a_i - a_j) q_i q_j (\gamma_i - \gamma_j) \right) \right] \\
        &= g(X_k, y_k)\sum_{j>1}  q_1(X_k) q_j(X_k) \frac{\eta_0}{4 n T} > 0,
    \end{split}
    \end{equation*}
which concludes the proof. 
\end{proof}

Next, we show that, if the directional limit exists, then it must select all completely positive/negative tokens. 
\begin{lemma}
    \label{lem:1_impt_norm2}
    Under Assumption \ref{asm:complete_token}, for any $\delta>0,$ by picking \begin{equation*}
      \eta_0 \geq 4n^2T^2, \quad d \geq \max\left\{ 256, \left(2 \log\frac{|\cS|^2}{\delta} \right)^2, (88 \eta_0^2+111\eta_0+2)^8  \right\},
    \end{equation*} with probability at least $1-\delta$, if $p_\infty = \lim_{t\to\infty} \frac{p_t}{\|p_t\|_2}$ exists, then $p_\infty$ satisfies \begin{equation*}
        s_*^X \in \cS_X(p_\infty),\qquad  \text{ for all } X \in \mathcal X_n,
    \end{equation*}
    where $s_*^X$ denotes the unique completely positive/negative token in the sequence $X$.
\end{lemma}

\begin{proof}
    We prove the lemma by contradiction. W.l.o.g., assume by contradiction that there exists $X \in \cX_n$ cointaining the important token $x_1$ s.t.\ $x_1 \notin \cS_X(p_\infty).$  We show that there exists $\bar{t}$ such that, for all $t \geq \bar{t},$ \begin{equation*}
        \cL(\mE^1, p^t) > \cL(\mE^1, p^1),
    \end{equation*} which contradicts the fact that the gradient flow always decreases the loss.

    To see this, we first note that by the definition of $\cS_X(p_\infty),$ there exists some $\tau > 0$ independent of  $t$ such that \begin{equation*}
        \min_{j \neq 1}p_\infty^\top (E_{x_1} - E_{x_j}) = - \tau. 
    \end{equation*} W.l.o.g, we assume that $x_2$ is the token that achieves the minimum.

    As $\lim_{t\to\infty}\|p_t\|_2 = \infty$ and $\lim_{t \rightarrow \infty} \frac{p_t}{\|p_t\|_2} = p_\infty,$ we have that, for any $\mu > 0, R > 0,$ there exists a large enough $\bar{t}$ such that \begin{equation*}
        \|p_t\|_2 \geq 2 R, \quad \left\| \frac{p_t}{\|p_t\|_2} - p_\infty\right\|_2 \leq \mu, \qquad \text{ for all }t \geq \bar{t}.
        \end{equation*}
    Thus, we have: \begin{equation*}
        \begin{split}
            \frac{p_t^\top}{\|p_t\|_2} (E_{x_1} - E_{x_2})  &= p_\infty^\top (E_{x_1} - E_{x_2}) + \left(  \frac{p_t}{\|p_t\|_2} - p_\infty\right )^\top (E_{x_1} - E_{x_2}) \\
            & \leq  -\tau + 2 \mu (4 \eta_0 +2)^2,
        \end{split}
    \end{equation*}
    where we have used the result of Lemma \ref{lem:bd_embd}.
Thus, by picking $\mu = \frac{\tau}{4 (4\eta_0 + 2)^2},$ we have \begin{equation*}
        \frac{p_t^\top}{\|p_t\|_2} (E_{x_1} - E_{x_2}) \leq -\frac{\tau}{2}, 
    \end{equation*} which implies that \begin{equation*}
        p_t^\top (E_{x_1} - E_{x_2}) \leq - \tau R.
    \end{equation*}

    Next, we upper bound $yf(X; p_t, \mE^1).$ We first note that \begin{equation*}
        \frac{q_{1}}{q_{2}} = \exp( p_t^\top (E_{x_1} - E_{x_2})) \leq \exp(-\tau R),
    \end{equation*} which gives \begin{equation*}
        q_1 \leq \exp(-\tau R).
    \end{equation*}
Note that \begin{equation*}
    \begin{split}
        yf(X; p_t, \mE^1) &= \sum_{i=1}^T q_{i} \gamma_i \\
        & \leq \exp(-\tau R) \gamma_1 + \max_{j\neq 1} \gamma_j \\
        & \leq  \exp(-\tau R)\left(\frac{\eta_0}{2} + (1+11\eta_0)d^{-\frac{1}{4}} \right) + (1+11\eta_0) d^{-\frac{1}{4}} .
    \end{split}
    \end{equation*}
Thus, by picking $R \geq \frac{\log d}{4 \tau},$ we have \begin{equation*}
        yf(X; p_t, \mE^1) \leq  \left(\frac{3}{2}+\frac{23}{2}\eta_0
        \right) d^{-\frac{1}{4}}\leq \frac{3}{4}d^{-\frac{1}{8}},
   \end{equation*} which implies a lower bound on the loss: \begin{equation}\label{eq:lbcon}
       \cL(\mE^1, p_t) \geq \frac{1}{n} \log(1 + \exp(- yf(X; p_t, \mE^1))) \geq \frac{1}{n} \log(1 + \exp(-\frac{3}{4}d^{-\frac{1}{8}}))\geq \frac{1}{2n},
   \end{equation}
where we used that $d\ge 256$ in the last passage.
    Under Assumption \ref{asm:complete_token}, by Lemma \ref{lem:overlap_v}, we have that $y \alpha_{x_i}\ge 1/(nT)$ if $x_i$ is either the completely positive or the completely negative token in $X$, and otherwise $y \alpha_{x_i}=0$. 
    Hence, given that each sequence $X$ contains a completely positive or negative token, we have that \begin{equation*}
        \frac{1}{T} \sum_{i=1}^T y \alpha_{x_i} \geq \frac{1}{n T^2}.
    \end{equation*}
  As $\eta_0 > 4n^2T^2>\sqrt{2 n T^2/11}$, by applying Lemma \ref{lem:un_loss_1step}, we obtain \begin{equation*}
       \begin{split}
            \cL(\mE^1, p_1) \leq \log(1 + \exp(-\frac{\eta_0}{4 n T^2}))\leq \log(1 + \exp(-n))\leq \exp(-n)<\frac{1}{2n}, 
       \end{split}
   \end{equation*}
which gives a contradiction and concludes the proof.
\end{proof}

Finally, we show that for each possible selection, if $p_t$ converges in direction, it must converge to the max-margin solution. In particular, we first prove the following lemma which gives an approximation to the directional gradient of the locally optimal selection. To do so, we define the secondary selection set and the locally optimal selection as follows: 
\begin{definition}
    \label{def:sec_select}
      Given a vector $p$, for each sequence $X,$ denote by $\cS^2_X(p)$ the secondary selection set given by \begin{equation}
            \cS^2_X(p) = \argmax \{ s: p^\top E_s, s \notin \cS_X(p) \}.
    \end{equation}
    We also denote by $\cS_X^<(p)$ the set of tokens that are not chosen in the first and in the second place, i.e.,
    \begin{equation}
        \cS^<_X(p) = X \setminus ( \cS_X(p) \bigcup \cS^2_X(p) ).
    \end{equation}
\end{definition}
    
\begin{definition}
    \label{def:local_opt}
    Given a vector $p,$ we say that $p$ is locally optimal if for every  $(X,y)$ pair,  we have \begin{equation*}
        \sum_{i \in \cS_X(p)} (\gamma_i(X,y) - \gamma_j(X,y) ) \geq \mu > 0, \qquad \text{ for all } j \in \cS_X^2(p),
    \end{equation*} for some constant $\mu$ that does not depends on $p.$
\end{definition} 

In the definition above and for the rest of this appendix, to help readability, we will abuse notation by letting indices (e.g., $i, j$ above) also denote the corresponding tokens (e.g., $x_i, x_j$ above).

\begin{lemma}
    \label{lem:grad_approx}
    Let $\overline{p}$ be a unit-norm vector and $p = R \overline{p}$ for some positive constant $R.$ Suppose $\overline{p}$ is a locally optimal direction as defined in Definition \ref{def:local_opt} with some $\mu$ that does not depends on $R.$ Moreover, suppose there exists a constant $\tau_1$ that may depend on $\overline{p}, \eta_0, n,T,d$ but not on $R,$ such that: \begin{equation}
    \begin{split}
        &\min_X \{\overline{p}^\top(E_s - E_{s'}), \forall s \in \cS_X(p), \forall s' \in \cS^2_X(p)\} \geq \tau_1, \\
        & \min_X \{\overline{p}^\top(E_s - E_{s'}), \forall s \in \cS^2_X(p), \forall s' \in \cS^<_X(p)\} \geq \tau_1.
    \end{split}
    \end{equation}
    Then, for any $\eps > 0$, for any $\hat{p} \approxeq p$ such that $\|\hat{p}\|_2 $ does not depend on $R$ and \begin{equation*}
    \begin{split}
         &\min_X \{\hat{p}^\top(E_s - E_{s'}), \forall s \in \cS_X(\hat{p}), \forall s' \in X \setminus \cS_X(\hat{p})\} \geq \tau_2, 
    \end{split}
    \end{equation*} 
    there exists $R$ large enough such that: \begin{equation*}
        \begin{split}
            & - \hat{p}^\top \nabla \cL(\mE^1, p) \leq (1 + \eps) \hbE\left[ \sum_{i \in \cS_X(p)}\sum_{j \in \cS^2_X(p)} (\widehat{a}_i(X) - \widehat{a}_j(X)) h_{i,j}(X,y,p)\right], \\
            & - \hat{p}^\top \nabla \cL(\mE^1, p) \geq (1 - \eps) \hbE\left[ \sum_{i \in \cS_X(p)}\sum_{j \in \cS^2_X(p)} (\widehat{a}_i(X) - \widehat{a}_j(X)) h_{i,j}(X,y,p)\right],
        \end{split}
    \end{equation*} where $\widehat{a}_i(X) = \hat{p}^\top E_{x_i}, \widehat{a}_j(X) = \hat{p}^\top E_{x_j}$ and \begin{equation*}
        h_{i,j}(X,y,p) =g(X,y) q_{i}(X) q_j(X)(\gamma_i(X,y) - \gamma_j(X,y)).
    \end{equation*}
\end{lemma}

\begin{proof}
    By Lemma \ref{lem:dir_grad}, we can write the directional gradient as follows: \begin{align}
    - \widehat{p}^\top \nabla_p \cL(\mE^1,p) & = \hbE\left[ \sum_{i=1}^T \sum_{j>i} (\widehat{a}_i(X) - \widehat{a}_j(X))h_{i,j}(X,y,p))\right] \notag \\
    & = \hbE\left[ \sum_{i \in \cS_X(p)}\sum_{j \in \cS^2_X(p)} (\widehat{a}_i(X) - \widehat{a}_j(X)) h_{i,j}(X,y,p)\right]\tag{B0} \label{eqn:b0}\\
    & \quad + \hbE\left[ \sum_{i \in \cS_X(p)} \sum_{j \in \cS^<_X(p)}   (\widehat{a}_i(X) - \widehat{a}_j(X)) h_{i,j}(X,y,p) ) \right]\tag{B1} \label{eqn:b1} \\
    & \quad +  \hbE\left[ \sum_{i \in X \setminus \cS_X(p)} \sum_{j > i: j \in X \setminus \cS_X(p)}   (\widehat{a}_i(X) - \widehat{a}_j(X)) h_{i,j}(X,y,p) ) \right] \tag{B2} \label{eqn:b2}.
    \end{align}

    The rest of the proof is to show that \begin{equation*}
        \begin{split}
            & - C_1 \exp( -\tau_1 R) \eqref{eqn:b0}\leq \eqref{eqn:b1} \leq C_1 \exp( -\tau_1 R)  \eqref{eqn:b0}, \\
            & - C_2 \exp( -\tau_1 R) \eqref{eqn:b0}\leq \eqref{eqn:b2} \leq C_2 \exp( -\tau_1 R)  \eqref{eqn:b0},
        \end{split}
    \end{equation*} for some $C_1, C_2 > 0$ that do not depend on $R.$ Then, by taking $R$ large enough, we obtain the desired result. 

    First, we simplify \eqref{eqn:b0}. 
    Note that, for all $i, i_0 \in \cS_X(p)$, we have that $\widehat{a}_i(X)= \widehat{a}_{i_0}(X)$. Hence, by switching the order of $i,j$, we obtain \begin{equation*}
        \begin{split}
            \sum_{i \in \cS_X(p)}\sum_{j \in \cS^2_X(p)} (\widehat{a}_{i}(X) - \widehat{a}_j(X)) h_{i,j}(X,y,p) &= \sum_{j \in \cS^2_X(p)} (\widehat{a}_{i_0}(X) - \widehat{a}_j(X)) \sum_{i \in \cS_X(p)} h_{i,j}(X,y,p) \\
            & \hspace{-11.5em}=g(X, y) \sum_{j \in \cS^2_X(p)} (\widehat{a}_{i_0}(X) - \widehat{a}_j(X)) q_{i_0}(X) q_j(X) \sum_{i \in \cS_X(p)}( \gamma_i(X,y) - \gamma_j(X,y)),
        \end{split}
    \end{equation*} 
for any $i_0\in \mathcal S_X(p)$.
    Since $p$ is a locally optimal direction, we have \begin{equation*}
        \sum_{i \in \cS_X(p)}( \gamma_i(X,y) - \gamma_j(X,y)) \geq \mu,\qquad  \text{for all } j \in \cS^2_X(p).
    \end{equation*}
    Now, we compare \eqref{eqn:b1} and \eqref{eqn:b0}. By the exact same reason above, we can rewrite \begin{equation*}
    \begin{split}
        \sum_{i \in \cS_X(p)}\sum_{j \in \cS^<_X(p)}& (\widehat{a}_i(X) - \widehat{a}_j(X)) h_{i,j}(X,y,p) \\
        &= g(X, y)\sum_{j \in \cS^<_X(p)} (\widehat{a}_{i_0}(X) - \widehat{a}_j(X)) q_{i_0}(X) q_j(X) \sum_{i \in \cS_X(p)}( \gamma_i(X,y) - \gamma_j(X,y)),
        \end{split}
    \end{equation*}
    for any $i_0\in \mathcal S_X(p)$, and we compare to \eqref{eqn:b0} term-by-term. Namely, for any $X$, $j \in \cS_X^2(p)$ and $k \in \cS^<_X(p),$ we have: \begin{align}
            & \frac{|\widehat{a}_{i_0}(X) - \widehat{a}_k(X) |}{\widehat{a}_{i_0}(X) - \widehat{a}_j(X)} \leq \frac{\|\hat{p}\|_2\|E_{x_{i_0}} - E_{x_j}\|_2}{\tau_2} \leq \frac{2\|\hat{p}\|_2 \max_s \|E_s\|_2}{\tau_2}:= C_3, \label{eq:a_i-a_j} \\
            & \frac{q_k(X)}{q_j(X)} = \exp( a_k(X) - a_j(X) ) \leq \exp( - \tau_1 R), \\
            & \frac{\sum_{i \in \cS_X(p)}| \gamma_i(X,y) - \gamma_k(X,y) |}{\sum_{i \in \cS_X(p)}( \gamma_i(X,y) - \gamma_j(X,y))} \leq \frac{2 T \max_s |\gamma_s|}{\mu} \leq  \frac{ 2T \max_s \|E_s\|_2}{\mu} := C_4 \label{eq:gamma_i-gamma_j},
        \end{align} which implies that, for any $X$, $j \in \cS_X^2(p)$ and $k \in \cS^<_X(p),$ \begin{equation*}
    \begin{split}
         |\widehat{a}_{i_0}(X) & - \widehat{a}_k(X)| q_{i_0}(X) q_k(X) \sum_{i \in \cS_X(p)}| \gamma_i(X,y) - \gamma_k(X,y)| \\
         \leq &\exp(-\tau_1 R) C_3 C_4 (\widehat{a}_{i_0}(X) - \widehat{a}_j(X)) q_{i_0}(X) q_j(X) \sum_{i \in \cS_X(p)}( \gamma_i(X,y) - \gamma_j(X,y)).
    \end{split}
    \end{equation*}

    Thus, we get that: \begin{equation*}
        |\eqref{eqn:b1}| \leq \exp(- \tau_1 R) T C_3 C_4  |\eqref{eqn:b0}|.
    \end{equation*}

    Next, we compare \eqref{eqn:b2} and \eqref{eqn:b0}. Take any $i' \in X\setminus \cS_X(p), k > i' \in X\setminus\cS_X(p), i_0 \in \cS_X(p), j \in \cS^2_X(p)$. We compare \begin{equation*}
        (\widehat{a}_{i'}(X) - \widehat{a}_k(X)) h_{i',k}(X,y,p)
    \end{equation*}  with each term in \eqref{eqn:b1}. We note that the bounds on $\frac{\widehat{a}_{i'}(X) - \widehat{a}_k(X)}{\widehat{a}_{i_0}(X) - \widehat{a}_j(X)}$ and $\frac{|\gamma_{i'}(X) - \gamma_k(X)|}{\sum_{i \in \cS_X(p)}(\gamma_i(X,y) - \gamma_j(X,y))}$ are the same as those in  \eqref{eq:a_i-a_j} and \eqref{eq:gamma_i-gamma_j}. Furthermore, \begin{equation*}
        \begin{split}
            \frac{q_{i'} q_k}{q_{i_0} q_j} \leq \exp(-\tau_1 R),
        \end{split}
    \end{equation*}
which gives that \begin{equation*}
        |\eqref{eqn:b2}| \leq T^2 \exp(- \tau_1 R) C_3 C_4 |\eqref{eqn:b0}|, 
    \end{equation*} thus concluding the proof.

\end{proof}

\begin{lemma}
\label{lem:max-margin}
     Under Assumption \ref{asm:complete_token}, for any $\delta>0,$ by picking \begin{equation*}
      \eta_0 \geq 4 n^2 T^2, \quad d \geq \max\left\{256,  \left(2 \log\frac{|\cS|^2}{\delta} \right)^2, (88 \eta_0^2+111\eta_0+2)^8  \right\},
    \end{equation*} with probability $\geq 1-\delta$ over the initialization, if $p_\infty = \lim_{t\to\infty} \frac{p_t}{\|p_t\|_2}$ exists, then  $p_\infty \in \cP_*(p_\infty)$.
\end{lemma}

\begin{proof}
    We prove the lemma by contradiction. We first assume that there exists $p_\infty$ such that $p_\infty \notin \cP_*(p_\infty)$ and $p_\infty = \lim_{t\to\infty} \frac{p_t}{\|p_t\|_2}.$ Then, we show that there exists $\widehat{p} \in \cP_*(p_\infty)$ such that, for any $\eps > 0$, there is $\bar{t}(\epsilon)$ ensuring 
    \begin{equation*}
        -\frac{\widehat{p}^\top}{\|\widehat{p}\|_2}\nabla_p \cL(\mE^1, p_t) \geq -(1-\eps)\frac{p_t^\top}{\|p_t\|_2}\nabla_p \cL(\mE^1, p_t),\qquad \text{ for all }t\ge \bar t(\epsilon).
    \end{equation*}
As a consequence, by Lemma \ref{lem:converge}, we have that $p_\infty = \frac{\widehat{p}}{\|\widehat{p}\|_2},$ which gives a contradiction. 

    For the rest of the proof, we  fix any $\eps > 0$ and denote $R = \|p_t\|_2$. 
    We define $\overline{p_t} = \frac{p_t \|\widehat{p}\|_2}{\|p_t\|_2},$ and we equivalently show that: \begin{equation}
    \label{eq:1impt_gradineq}
        - \widehat{p}^\top \nabla_p \cL(\mE^1, p_t) \geq - (1-\eps) \overline{p_t} \nabla_p \cL(\mE^1, p_t).
    \end{equation}
    To prove this, we first note that since $p_\infty \notin \cP_*(p_\infty),$ for all $\frac{\widehat{p}}{\|\widehat{p}\|_2} \in \cP_*(p_\infty),$ there exists $\tau_0$ independent of $R$ such that \begin{equation*}
        \left\| \widehat{p}  - p_\infty \|\widehat{p}\|_2 \right\|_2 \geq \tau_0.
    \end{equation*}

    Thus, by the definition of $\cP_*(p_\infty)$, there exists $\cX_0 \subseteq \cX_n$ such that for each sequence $X\in \cX_0$, we can find a pair of indices $(i,j)$ with $i \in \cS_X(p_\infty), j \in X\setminus \cS_X(p_\infty)$ violating the margin, i.e., \begin{equation*}
         (\|\widehat{p}\|_2 p_\infty)^\top (E_{x_i} - E_{x_j}) \leq 1 - 3\tau,
    \end{equation*} for some $\tau < \frac{1}{6}$ that does not depend on $R.$ 
    With a slight abuse of notation, we define $\tau$ as \begin{equation*}
    \begin{split}
        \tau = &\frac{1}{3} \min\{ \min_{X \in \cX_0}\{ 1-(\|\widehat{p}\|_2 p_\infty)^\top (E_{x_i} - E_{x_j}), i \in \cS_X(p_\infty), j \in  \cS_X^2(p_\infty) \}, \\
        & \min_{X \in \cX_n} \{ (\|\widehat{p}\|_2 p_\infty)^\top (E_{x_i} - E_{x_j}), i \in \cS_X(p_\infty), j \in  \cS_X^2(p_\infty) \}, \\
        & \min_{X \in \cX_n} \{ (\|\widehat{p}\|_2 p_\infty)^\top (E_{x_i} - E_{x_j}), i \in \cS^2_X(p_\infty), j \in  \cS_X^<(p_\infty) \}\}.
    \end{split}
    \end{equation*}    
    This means that, for all $X \in \cX_n$ and for all $(i,j)$ pairs such that $i \in \cS_X(p_\infty), j \in \cS^2_X(p_\infty)$, we have \begin{equation*}
        (\|\widehat{p}\|_2 p_\infty)^\top (E_{x_i} - E_{x_j}) \geq 3\tau;
    \end{equation*} for all pairs $(i, j)$ such that $i \in \cS^2_X(p_\infty), j \in \cS^<_X(p_\infty)$, we have \begin{equation*}
        (\|\widehat{p}\|_2 p_\infty)^\top (E_{x_i} - E_{x_j}) \geq 3 \tau;
    \end{equation*} and for all $X \in \cX_0,i \in \cS_X(p_\infty), j \in \cS^2_X(p_\infty),$ we have \begin{equation*}
        (\|\widehat{p}\|_2 p_\infty)^\top (E_{x_i} - E_{x_j}) \leq 1- 3 \tau,
    \end{equation*} with some $\tau$ that does not depend on $R.$

    Now, we compute the overlap with $\overline{p_t}.$ For all $X$ and $(i,j)$, we have \begin{equation*}
        \begin{split}
            \overline{p_t}^\top (E_{x_i} - E_{x_j}) &=  (\|\widehat{p}\|_2 p_\infty)^\top (E_{x_i} - E_{x_j}) + (\overline{p_t} - \|\widehat{p}\|_2 p_\infty)^\top(E_{x_i} - E_{x_j}).
        \end{split}
    \end{equation*}
    We upper bound \begin{equation*}
       |(\overline{p_t} - \|\widehat{p}\|_2 p_\infty)^\top(E_{x_i} - E_{x_j}) | \leq  \|\widehat{p}\|_2 \left\| \frac{p_t}{\|p_t\|_2} - p_\infty\right\|_2 \|E_{x_1} - E_{x_2}\|_2, 
    \end{equation*} and since $\|\widehat{p}\|_2, \|E_{x_1} - E_{x_2}\|_2$ are finite, we have \begin{equation*}
        \lim_{t \rightarrow \infty}|(\overline{p_t} - \|\widehat{p}\|_2 p_\infty)^\top(E_{x_i} - E_{x_j}) | = 0.
    \end{equation*}
    Thus, we can pick $t_1,$ such that for $t \geq t_1,$ we have \begin{equation*}
        |(\overline{p_t} - \|\widehat{p}\|_2 p_\infty)^\top(E_{x_i} - E_{x_j}) | \leq \tau,
    \end{equation*} which implies that, for all $X \in \cX_n$ and for all $(i,j)$ pairs such that $i \in \cS_X(p_\infty), j \in \cS^2_X(p_\infty)$, we have \begin{equation*}
        \overline{p_t}^\top (E_{x_i} - E_{x_j}) \geq \tau;
    \end{equation*} for all $(i, j)$ pairs such that $i \in \cS^2_X(p_\infty), j \in \cS^<_X(p_\infty)$, we have \begin{equation*}
        \overline{p_t}^\top (E_{x_i} - E_{x_j}) \geq  \tau;
    \end{equation*} and for all $X \in \cX_0,i \in \cS_X(p_\infty), j \in \cS^2_X(p_\infty),$ we have: \begin{equation*}
        \overline{p_t}^\top (E_{x_i} - E_{x_j}) \leq 1-  \tau,
    \end{equation*} for some $\tau$ that does not depend on $R.$

    Next, we show that $\overline{p_t}$ is a locally optimal solution as per Definition \ref{def:local_opt}. By Lemma \ref{lem:1_impt_norm}, $p_\infty$ selects all the completely positive/negative tokens. Thus, as $\overline{p_t} \approxeq p_\infty$,  $\overline{p_t}$ also selects such tokens, the rest being irrelevant by Assumption \ref{asm:complete_token}. Hence, for any pair $(X,y)$ and for any $j \in X \setminus \cS_X(p_t),$ we have: \begin{equation*}
        \sum_{i \in \cS_X(p_t)} (\gamma_i(X,y) - \gamma_j(X,y)) \geq \frac{\eta_0}{4nT},
    \end{equation*} by picking $d$ large enough (as per the hypothesis of the lemma). By construction, $\hat{p} \approxeq p_t$, $\|\hat{p}\|_2$ does not depends on $R$ and,  moreover, for any $X$, \begin{equation*}
        \begin{split}
            \hat{p}^\top (E_{x_i} - E_{x_j}) \geq 1, \qquad \text{ for all } i \in \cS_X(\hat{p}),\,\, j \in X \setminus \cS_X(\hat{p}) .
        \end{split}
    \end{equation*}
    By applying Lemma \ref{lem:grad_approx} on both $\hat{p}$ and $\overline{p_t},$ we have that for any $\eps_1 > 0$ there exist $t_2$ s.t.\ for all $t \geq \max \{t_1, t_2\},$ we have \begin{equation*}
        \begin{split}
            & - \hat{p}^\top \nabla_p \cL(\mE^1, p_t) \geq (1 - \eps_1) \hbE\left[ \sum_{i \in \cS_X(p)}\sum_{j \in \cS^2_X(p)} (\widehat{a}_i(X) - \widehat{a}_j(X)) h_{i,j}(X,y,p_t)\right], \\
            & - \overline{p_t}^\top \nabla_p \cL(\mE^1, p_t) \leq (1 + \eps_1) \hbE\left[ \sum_{i \in \cS_X(p)}\sum_{j \in \cS^2_X(p)} (\overline{a_i}(X) - \overline{a_j}(X)) h_{i,j}(X,y,p_t)\right],
        \end{split}
    \end{equation*}
where $\overline{a_i}(X), \overline{a_j}(X)$ are defined analogously to $\widehat{a}_i(X), \widehat{a}_j(X)$ by replacing $\hat{p}$ with $\overline{p_t}$.

    Now, we further show that, 
    for any $\eps_2 > 0,$ there exist $t_3$ such that for all $t \geq t_3$, \begin{equation*}
    \begin{split}
        &\hbE\left[ \sum_{i \in \cS_X(p)}\sum_{j \in \cS^2_X(p)} (\overline{a_i}(X) - \overline{a_j}(X)) h_{i,j}(X,y,p_t)\right]\\
        &\hspace{4em}\leq (1+\eps_2) \hbE_{X \in \cX_0}\left[ \sum_{i \in \cS_X(p)}\sum_{j \in \cS^2_X(p)} (\overline{a_i}(X) - \overline{a_j}(X)) h_{i,j}(X,y,p_t)\right].
    \end{split}
    \end{equation*} 

    To see this, we use the same 
    idea as in the proof of Lemma \ref{lem:grad_approx}. We can write \begin{align}
       & \hbE\left[ \sum_{i \in \cS_X(p)}\sum_{j \in \cS^2_X(p)} (\overline{a_i}(X) - \overline{a_j}(X)) h_{i,j}(X,y,p_t)\right]\notag \\
       &\hspace{4em}= \hbE_{X \in \cX_0}\left[ \sum_{i \in \cS_X(p)}\sum_{j \in \cS^2_X(p)} (\overline{a_i}(X) - \overline{a_j}(X)) h_{i,j}(X,y,p_t)\right] \label{eq:X0} \tag{A0}\\
        &\hspace{4em} + \hbE_{X' \in \cX_n \setminus \cX_0}\left[ \sum_{i \in \cS_{X'}(p)}\sum_{j \in \cS^2_{X'}(p)} (\overline{a_i}(X') - \overline{a_j}(X')) h_{i,j}(X',y',p_t)\right] \label{eq:X-X0} \tag{A1},
    \end{align} and it is sufficient to show that \begin{equation*}
        \eqref{eq:X-X0} \leq \eps_2 \eqref{eq:X0}.
    \end{equation*}

    To prove this, we compare term-by-term. Let $X \in \cX_0, X' \in \cX_n \setminus \cX_0, j \in \cS^2_X(p_t), j' \in \cS^2_{X}(p_t),$ and recall that: \begin{equation*}
        \begin{split}
            &\sum_{i \in \cS_X(p_t)} (\overline{a_i}(X) - \overline{a_j}(X)) h_{i,j}(X,y,p_t) \\
            &\hspace{4em}= g(X,y) (\overline{a_{i_0}}(X) - \overline{a_j}(X)) q_{i_0}(X) q_j(X) \sum_{i \in \cS_X(p_t)}(\gamma_i(X,y) - \gamma_j(X,y)),
            \end{split}
            \end{equation*}
            \begin{equation*}
                \begin{split}
            &\sum_{i \in \cS_{X'}(p_t)} (\overline{a_i}(X') - \overline{a_{j'}}(X')) h_{i,j'}(X',y',p_t) \\
            &\hspace{4em}= g(X',y') (\overline{a_{i_1}}(X') - \overline{a_{j'}}(X')) q_{i_1}(X') q_{j'}(X') \sum_{i \in \cS_{X'}(p_t)}(\gamma_i(X',y') - \gamma_{j'}(X',y')),
        \end{split}
    \end{equation*}
for any $i_0\in \mathcal S_X(p_t), i_1\in \mathcal S_{X'}(p_t)$. Note that \begin{equation}
\label{eq:gX-gX'}
        \frac{g(X',y')}{g(X,y)} \leq \frac{\max_{X,y} g(X,y)}{\min_{X,y} g(X,y) } \leq  \max_{X,y} (1+ \exp(yf(X))) \leq (1+ \exp(\eta_0)) := C_5.
    \end{equation}

    By using the same argument as in \eqref{eq:a_i-a_j} and \eqref{eq:gamma_i-gamma_j}, we have \begin{equation*}
        \begin{split}
            & \frac{\overline{a_{i_1}}(X') - \overline{a_{j'}}(X')}{\overline{a_{i_0}}(X) - \overline{a_j}(X)} \leq C_3, \\
            & \frac{\sum_{i \in \cS_{X'}(p_t)}(\gamma_i(X',y') - \gamma_{j'}(X',y'))}{\sum_{i \in \cS_X(p_t)}(\gamma_i(X,y) - \gamma_j(X,y))} \leq C_4.
        \end{split}
    \end{equation*}
    Finally, we need to upper bound: \begin{equation*}
        \frac{q_{i_1}(X') q_{j'}(X')}{q_{i_0}(X) q_j(X)}.
    \end{equation*}
    We note that\begin{equation*}
        \begin{split}
            a_{i_1}(X') - a_{j'}(X') \geq R/\|\hat{p}\|_2, \\
            a_{i_0}(X) - a_{j}(X) \leq (1 - \tau)R/\|\hat{p}\|_2,
        \end{split}
    \end{equation*} where $a_i(X) = p_t^\top E_{x_i}.$ Thus by Lemma \ref{lem:bds_on_q}, we have: \begin{equation*}
        q_{i_0}(X) \geq \frac{1}{T},\quad q_j(X) \geq \frac{1}{T\exp((1 - \tau)R/\|\hat{p}\|_2)},\quad  q_{i_1}(X') \leq 1, \quad q_{j'}(X') \leq \frac{1}{\exp(R/\|\hat{p}\|_2)},
    \end{equation*}
which implies that \begin{equation*}
        \frac{q_{i_1}(X') q_{j'}(X')}{q_{i_0}(X) q_j(X)} \leq T^2 \exp(- \tau R/\|\hat{p}\|_2).
    \end{equation*} Thus, for each $X \in \cX_0, X' \in \cX_n \setminus \cX_0, j \in \cS^2_X(p_t), j' \in \cS^2_{X}(p_t)$, we have \begin{equation*}
\sum_{i \in \cS_{X'}(p)} (\overline{a_i}(X') - \overline{a_{j'}}(X')) h_{i,j'}(X',y',p_t)
         \leq C_6     \exp(- \tau R/\|\hat{p}\|_2) \sum_{i \in \cS_X(p)} (\overline{a_i}(X) - \overline{a_j}(X)) h_{i,j}(X,y,p_t).
    \end{equation*}
    Thus by picking large enough $t_3$ which gives  large enough $R, $ we have: \begin{equation*}
        \eqref{eq:X-X0} \leq \eps_2 \eqref{eq:X0}.
    \end{equation*}
This allows us to conclude that
        \begin{equation*}
        \begin{split}
             - \hat{p}^\top \nabla_p \cL(\mE^1, p_t) &\geq (1 - \eps_1) \hbE\left[ \sum_{i \in \cS_X(p)}\sum_{j \in \cS^2_X(p)} (\widehat{a}_i(X) - \widehat{a}_j(X)) h_{i,j}(X,y,p_t)\right] \\
             &\geq (1 - \eps_1)\hbE_{X \in \cX_0}\left[ \sum_{i \in \cS_X(p)}\sum_{j \in \cS^2_X(p)} (\widehat{a}_i(X) - \widehat{a}_j(X)) h_{i,j}(X,y,p_t)\right], \\
             -\overline{p_t}^\top \nabla_p \cL(\mE^1, p_t)  &\leq (1 + \eps_1)(1+\eps_2) \hbE_{X \in \cX_0}\left[ \sum_{i \in \cS_X(p)}\sum_{j \in \cS^2_X(p)} (\overline{a_i}(X) - \overline{a_j}(X)) h_{i,j}(X,y,p_t)\right].
        \end{split}
    \end{equation*} Note that, for each $X \in \cX_0$, \begin{equation*}
        \widehat{a}_i(X) - \widehat{a}_j(X) \geq 1,\qquad  \overline{a_i}(X) - \overline{a_j}(X) \leq 1-\tau,
    \end{equation*}
which gives that \begin{equation*}
         - \hat{p}^\top \nabla_p \cL(\mE^1, p_t) \geq - \frac{1-\eps_1}{(1+\eps_1)(1+\eps_2)(1-\tau)}\overline{p_t}^\top \nabla_p \cL(\mE^1, p_t).
    \end{equation*} Since $\eps_1, \eps_2$ can be arbitrarily small, the proof is complete.
\end{proof}


\subsection{Proof of Lemma \ref{lem:suff}}
\label{apx:pf_lem_suff}

We first prove two auxiliary lemmas
showing that for any max-margin solution, at least one constraint must be tight. 

\begin{lemma}
\label{lem:tight_one}
    Consider the max-margin problem in \eqref{eqn:max-margin}, and let $\hat{p}$ be the solution. Then, there exist $X \in \cX_n, s \in \cS_X(p), s' \in X \setminus \cS_X(p)$ such that: \begin{equation*}
        \hat{p}^\top (E_s - E_{s'}) = 1.
    \end{equation*} 
\end{lemma}

\begin{proof}
    Assume by contradiction that for all $X \in \cX_n, s \in \cS_X(p), s' \in X \setminus \cS_X(p), $   \begin{equation*}
        \hat{p}^\top (E_s - E_{s'}) > 1.
    \end{equation*}

    Define \begin{equation*}
        \tau = \min_{X \in \cX_n, s \in \cS_X(p), s' \in X \setminus \cS_X(p)} \hat{p}^\top (E_s - E_{s'}) ,
    \end{equation*} and we know that $\tau >1.$ Then, $\hat{p}' = \frac{\hat{p}}{\tau}$ is also a feasible solution to \eqref{eqn:max-margin} with $\|\hat{p}'\|_2 < \|\hat{p}\|_2,$ which contradicts to the fact that $\hat{p}$ is the solution to the max-margin problem.   
\end{proof}

\begin{lemma}
\label{lem:ub_mmnorm}
    Under the same condition on $d$ in Lemma \ref{lem:suff}, let $\hat{p}$ be the unique solution of the max-margin problem in \eqref{eqn:max-margin} that only selects the completely positive and negative tokens. Then we have: \begin{equation*}
        \| \hat{p} \|_2 \leq 4n.
    \end{equation*}
\end{lemma}

\begin{proof}
    By the definition of the max-margin problem, it is sufficient to construct a feasible $p$ with $\|p \|_2 \leq 4n$, which implies that: \begin{equation*}
        \|\hat{p}\|_2 \leq \|p\|_2 \leq 4n.
    \end{equation*}
    
    Since $\hat{p}$ only selects completely positive and negative tokens, we know that $(\bigcup_{X} \cS_X(\hat{p})) \bigcap (\bigcup_{X} \overline{\cS_X(\hat{p})}) = \emptyset.$ We let \begin{equation*}
        p = 2 \sum_{s \in \bigcup_{X} \cS_X(\hat{p})} E_s^0, 
    \end{equation*} and by \eqref{eq:cond} we know that $\|p\|_2 \leq 4n,$ since there are at most $n$ completely positive and negative tokens.

    It remains to show that $p$ is feasible with high probability. To see this, for $s \in \bigcup_{X} \cS_X(\hat{p})$, we have \begin{equation*}
        \begin{split}
        p^\top E_s^1 &= 2 + \sum_{s' \in \cS_X(\hat{p}):s' \neq s} (E_s^0)^\top E_{s'}^0 +  
 \frac{\eta_0}{2} \alpha_s p^\top v + p^\top \err_s \\
        & \geq  2 - \frac{\eta_0}{2} \alpha_s \frac{n}{\sqrt{d}} \sqrt{2 \log \frac{|\cS|^2}{\delta}} - 22\sqrt{2n} \eta_0 d^{-\frac{1}{4}} , \\
        &\geq  2 - \frac{\eta_0}{2}\frac{n}{\sqrt{d}} \sqrt{2 \log \frac{|\cS|^2}{\delta}} - 22\sqrt{2n} \eta_0 d^{-\frac{1}{4}}
       \end{split}
    \end{equation*}  where the first  inequality follows from \eqref{eq:cond} and  Lemma \ref{lem:overlap_v}. Similarly, for $s \in \bigcup_{X} \overline{\cS_X(\hat{p})},$ \begin{equation*}
        p^\top E_s^1 \leq \frac{\eta_0}{2}\frac{n}{\sqrt{d}} \sqrt{2 \log \frac{|\cS|^2}{\delta}} + 22\sqrt{2n} \eta_0 d^{-\frac{1}{4}}
    \end{equation*}

    Thus, by picking \begin{equation*}
        d \geq \left(2 \eta_0 \left( n \sqrt{2 \log \frac{|\cS|^2}{\delta}} + 44 \sqrt{2n}\right) \right)^4,
    \end{equation*} 
    we have $p^\top (E_s^1 - E_{s'}^1) \geq 1,$ with $s \in \bigcup_{X} \cS_X(\hat{p}), s' \in \bigcup_{X} \overline{\cS_X(\hat{p})},$ which indicates the feasibility of $p$ and finishes the proof.
\end{proof}

\begin{proof}[Proof of Lemma \ref{lem:suff}]
    Let $\hat{p}'$ be the max-margin solution of \eqref{eqn:max-margin} with a different selection. By Theorem \ref{thm:complete_seq}, we have that, for all $X, s_*^X \in \cS_X(\hat{p}')$. We denote by $i_*^X$ the index of $s_*^X$.
    Assume by contradiction $p_\infty = \frac{\hat{p}'}{\|\hat{p}'\|_2}$. We will now show that this implies the following statement: 
for any $\eps > 0$, there is a $t(\eps)$ ensuring \begin{equation}\label{eq:statm}
        - \frac{\hat{p}^\top}{\|\hat{p}\|_2} \nabla_p \cL(\mE, p_t) \geq - (1-\eps )\frac{p_t^\top}{\|p_t\|_2} \nabla_p \cL(\mE, p_t),\qquad \text { for all }t \geq t(\eps).
    \end{equation} Then,  by Lemma \ref{lem:converge}, we have that $p_\infty = \frac{\widehat{p}}{\|\widehat{p}\|_2},$ which gives a contradiction.
    
    As in the proof of Lemma \ref{lem:max-margin}, we define $\overline{p_t} = \frac{p_t}{\|p_t\|_2} \|\hat{p}\|_2$. Thus, \eqref{eq:statm} is equivalent to \begin{equation*}
        - \hat{p}^\top \nabla_p \cL(\mE, p_t) \geq - (1-\eps )\overline{p_t}^\top \nabla_p \cL(\mE, p_t).
    \end{equation*}

    First of all, since $\hat{p}, \hat{p}'$ are two max-margin solutions, by the definition of max-margin solution and Lemma \ref{lem:tight_one}, we have: \begin{equation*}
    \begin{split}
        &\hat{p}^\top (E_{s_*^X} - E_s) \geq 1, \qquad \forall s \in X \setminus s_*^X, \forall X \in \cX_n \\
        &(\hat{p}')^\top (E_{s} - E_{s'}) = 1, \qquad \exists X \in \cX_n,  s \in \cS_X(\hat{p}'), s' \in X \setminus \cS_X(\hat{p}'),
    \end{split}
    \end{equation*} which implies that \begin{equation*}
        \frac{\hat{p}'^\top \|\hat{p}\|_2}{\|\hat{p}'\|_2} (E_{s} - E_{s'}) = \frac{\|\hat{p}\|_2}{\|\hat{p}'\|_2} = 1-\mu <1 ,\qquad \exists X \in \cX_n,   s' \in X \setminus \cS_X(\hat{p}'),\,\, \forall s  \in \cS_X(\hat{p}').
    \end{equation*}

For simplicity, we define $ \overline{\cS_X^\circ(\hat{p}')} \subseteq X \setminus \cS_X(\hat{p}')$ such that \begin{equation*}
    (\hat{p}')^\top (E_{s} - E_{s'}) = 1, \qquad \forall X \in \cX_n,  s' \in  \overline{\cS_X^\circ(\hat{p}')} , s  \in \cS_X(\hat{p}'),
\end{equation*} and note that $\overline{\cS_X^\circ(\hat{p}')} $ can be empty for some $X.$ Also define \begin{equation*}
    \tau_0 = \min_{X \in \cX_n, s \in \cS_X(\hat{p}'), s' \in \overline{\cS_X(\hat{p}')} \setminus  \overline{\cS_X^\circ(\hat{p}')} } (\hat{p}')^\top (E_{s} - E_{s'}) -1.
\end{equation*}

As $\lim_{t \rightarrow \infty} \overline{p_t} = \frac{\hat{p}' \|\hat{p}\|_2}{\|\hat{p}'\|_2}$, for any $\eps_1 \in (0,\mu)$ small enough, there exists a $t_1$ ensuring the following for all $t \geq t_1$:  \begin{equation}
\begin{split}
\label{eq:sepa_impt}
    &\overline{p_t}^\top (E_{s} - E_{s'}) \leq 1-\mu + \eps_1 < (1-\mu)(1+\tau_0) - \eps_1, \qquad \forall X \in \cX_n, s \in \cS_X(\hat{p}'), s' \in  \overline{\cS_X^\circ(\hat{p}')}, \\
    & \overline{p_t}^\top (E_{s} - E_{s'}) \geq  (1-\mu)(1+\tau_0) - \eps_1, \qquad \forall X \in  \cX_n, s \in \cS_X(\hat{p}'), s' \in \overline{\cS_X(\hat{p}')} \setminus  \overline{\cS_X^\circ(\hat{p}')} , 
\end{split}
\end{equation} which implies that $\cS_X^2(p_t) \subseteq \overline{\cS_X^\circ(\hat{p}')},$ if $\overline{\cS_X^\circ(\hat{p}')} \neq \emptyset.$

By applying Lemma \ref{lem:grad_approx} to $\overline{p_t}$, we obtain that, for any $\eps_2 > 0,$ there exists a $t_2$ ensuring that, for all $t \geq t_2$, \begin{align}
        -\overline{p_t}^\top &\nabla_p \cL(\mE^1, p_t) \leq (1 + \eps_2) \hbE\left[g(X, y) \sum_{i \in \cS_X(p_t)} \sum_{j \in \cS^2_X(p_t)} ( \overline{a_i}(X) -  \overline{a_j}(X) q_i(X) q_j(X) (\gamma_i(X) - \gamma_j(X))  ) \right] \notag \\
        &\leq (1 + \eps_2) \hbE\left[g(X, y) \sum_{i \in \cS_X(p_t)} \sum_{j \in \cS^2_X(p_t)} ( \overline{a_i}(X) -  \overline{a_j}(X) ) q_i(X) q_j(X) |\gamma_i(X) - \gamma_j(X)| \right]\label{eq:approx_ptbar} \tag{C1} 
    \end{align}

Now we further approximate  \eqref{eq:approx_ptbar}. For simplicity, we define $\cX_\circ \in \cX_n$ such that,  for all $X \in \cX_\circ$, $\overline{\cS_X^\circ(\hat{p}')} \neq \emptyset$. We show that for any $\eps_3 > 0,$ there exists $t_3$ ensuring that, for all $t \geq \max\{t_1,t_2,t_3\},$ \begin{equation*}
    \eqref{eq:approx_ptbar} \leq (1 + \eps_2) (1+\eps_3) \frac{1}{n} \sum_{X \in \cX_\circ}  
    g(X, y) \sum_{i \in \cS_X(p_t)} \sum_{j \in \cS^2_X(p_t)} ( \overline{a_i}(X) -  \overline{a_j}(X) ) q_i(X) q_j(X) |\gamma_i(X) - \gamma_j(X)|
\end{equation*}

To see this, for any $\eps_3 > 0,$ we show that for $t \geq \max\{t_1, t_2, t_3\}$ and any $X \in \cX_\circ, X' \in \cX_n \setminus \cX_\circ$, 
\begin{equation}
\label{eq:further_approx}
    \begin{split}
        &g(X', y) \sum_{i \in \cS_X(p_t)} \sum_{j \in \cS^2_X(p_t)} ( \overline{a_{i_*^{X'}}}(X') -  \overline{a_j}(X') )q_{i_*^{X'}}(X') q_j(X') (\gamma_{i_*^{X'}}(X') - \gamma_j(X'))  ) \\
        \leq & \frac{1}{n}\eps_3 g(X, y) \sum_{i \in \cS_X(p_t)} \sum_{j \in \cS^2_X(p_t)} ( \overline{a_{i_*^{X}}}(X) -  \overline{a_j}(X) )q_{i_*^{X}}(X) q_j(X) (\gamma_{i_*^X}(X) - \gamma_j(X))  ).
    \end{split}
\end{equation}

Indeed, using the same methods as in \eqref{eq:a_i-a_j}, \eqref{eq:gamma_i-gamma_j}, \eqref{eq:gX-gX'}, we have \begin{equation*}
    \begin{split}
        & \frac{\overline{a_{i}}(X') -  \overline{a_j}(X')}{\overline{a_{i}}(X) -  \overline{a_j}(X)} \leq C_3, \quad \forall i \in \cS_X(p_t)\\
        & \frac{\gamma_{i}(X') - \gamma_j(X'))}{\gamma_{i}(X) - \gamma_j(X))} \leq C_4, \quad \forall i \in \cS_X(p_t) \\
        & \frac{g(X',y')}{g(X,y)} \leq C_5,
    \end{split}
\end{equation*} where $C_3, C_4, C_5$ are constants that do not depend on $R.$ It remains to upper bound $\frac{q_{i}(X') q_j(X')}{q_{i}(X) q_j(X)},$ which is equivalent to $\frac{q_{i_*^{X'}}(X') q_j(X')}{q_{i_*^{X}}(X) q_j(X)}$ as all $i \in \cS_X(p_t)$ has the same $q_i.$ 

By Lemma \ref{lem:bds_on_q} together with \eqref{eq:sepa_impt}, we have: \begin{equation*}
    \begin{split}
        &q_{i_*^{X'}}(X') q_j(X') \leq \exp( - ((1- \mu)(1+\tau_0) - \eps_1) R ),\\
        & q_{i_*^{X}}(X) q_j(X) \geq \frac{1}{T^2} \exp( -(1- \mu + \eps_1) R),
    \end{split} 
\end{equation*} where for simplicity we denote $R = \|p_t\|_2$. This implies \begin{equation*}
    \frac{q_{i_*^{X'}}(X') q_j(X')}{q_{i_*^{X}}(X) q_j(X)} \leq T^2 \exp( - ((1-\mu) \tau_0 - 2 \eps_1) R).
\end{equation*} Since $\|p_t\|_2 \rightarrow \infty,$ and $C_3, C_4, C_5$ do not depend on $R,$ we can pick $t_3$ large enough such that: \begin{equation*}
   n T^2 C_3 C_4 C_5 \exp( - ((1-\mu) \tau_0 - 2 \eps_1) R) \leq \eps_3,
\end{equation*} which implies \eqref{eq:further_approx}. Thus, we have \begin{equation}\label{eq:barpt}
    \begin{split}
          -&\overline{p_t}^\top \nabla_p\mathcal L(\mE^1, p_t) \\
          & \leq (1 + \eps_2) (1+\eps_3) \frac{1}{n} \sum_{X \in \cX_\circ}  
    g(X, y) \sum_{i \in \cS_X(p_t)} \sum_{j \in \cS^2_X(p_t)} ( \overline{a_i}(X) -  \overline{a_j}(X) ) q_i(X) q_j(X) |\gamma_i(X) - \gamma_j(X)| \\
        &=  (1 + \eps_2) (1+\eps_3) \frac{1}{n} \sum_{X \in \cX_\circ}  
    g(X, y)  \sum_{j \in \cS^2_X(p_t)} ( \overline{a_{i_*^X}}(X) -  \overline{a_j}(X) ) q_{i_*^X}(X) q_j(X) |\gamma_{i_*^X}(X) - \gamma_j(X)| \\
    & + (1 + \eps_2) (1+\eps_3) \frac{1}{n} \sum_{X \in \cX_\circ}  
    g(X, y) \sum_{i \in \cS_X(p_t): i \neq i_*^X} \sum_{j \in \cS^2_X(p_t)} ( \overline{a_{i}}(X) -  \overline{a_j}(X) ) q_{i}(X) q_j(X) |\gamma_{i}(X) - \gamma_j(X)|
    \end{split}
\end{equation} 
    We then compute by Lemma \ref{lem:dir_grad} that \begin{align}
            - &\hat{p}^\top \nabla_p\mathcal L(\mE^1, p_t)              =  \hbE\left[ g(X, y)\sum_{j \in  \overline{\cS_X(p_t)}} ( \widehat{a_{i_*^X}}(X) -  \widehat{a_j}(X) ) q_{i_*^X}(X) q_j(X) (\gamma_{i_*^X}(X) - \gamma_j(X))  )  \right] \label{eq:a_i_*} \tag{D1} \\
            & \quad + \hbE\left[ g(X, y) \sum_{i  \in \cS_X(p_t), i \neq i_*^X} \sum_{j \in  \overline{\cS_X(p_t)}} ( \widehat{a_{i}}(X) -  \widehat{a_j}(X) ) q_{i}(X) q_j(X)  (\gamma_{i}(X) - \gamma_j(X))  )  \right] \label{eq:a_impt-a_nonimpt} \tag{D2} \\
            & \quad + \hbE\left[ g(X, y) \sum_{i  \in \overline{\cS_X(p_t)}} \sum_{j \in  \overline{\cS_X(p_t)}:j>i} ( \widehat{a_{i}}(X) -  \widehat{a_j}(X) ) q_{i}(X) q_j(X)  (\gamma_{i}(X) - \gamma_j(X))  )  \right]. \label{eq:a_nonimpt-a_nonimpt} \tag{D3}
    \end{align}

    We show that by picking \begin{equation*}
        d \geq \left( \frac{2816 Tn^2  (1+2\eta_0)}{\mu }   \right)^4,
    \end{equation*}  we have \begin{equation*}
        \begin{split}
            |\eqref{eq:a_impt-a_nonimpt} |  \leq \frac{\mu}{2}  \eqref{eq:a_i_*},
         \end{split}
    \end{equation*}
    and for any $\eps_4>0,$ there exists $t_4$ such that for all $t \geq t_4,$ \begin{equation*}
        |\eqref{eq:a_nonimpt-a_nonimpt} |  \leq  \eps_4  \eqref{eq:a_i_*}. 
    \end{equation*}

    To show $|\eqref{eq:a_impt-a_nonimpt} |  \leq \frac{\mu}{2}  \eqref{eq:a_i_*},$ it suffices to note that, for any $j \in \overline{\cS_X(p_t)}$, \begin{equation*}
    \begin{split}
    \widehat{a_{i_*^X}}(X) -  \widehat{a_j}(X) &\geq 1,\\
        |\widehat{a_{i}}(X) -  \widehat{a_j}(X) | &\leq 16n  (1 + 2 \eta_0), \qquad \forall \,i,\\
         | \gamma_i(X) - \gamma_j(X) | &\leq 22 \eta_0 d ^{-1/4}, \qquad \forall\,i \neq i_*^X,\\
         \gamma_{i_*^X}(X) - \gamma_j(X) &\geq \frac{\eta_0}{4nT },
    \end{split}
    \end{equation*} 
    where the second inequality follows from Lemma \ref{lem:ub_mmnorm} and Lemma \ref{lem:bd_embd}. 
    To show $ |\eqref{eq:a_nonimpt-a_nonimpt} |  \leq  \eps_4  \eqref{eq:a_i_*},$ we use again the inequalities above and the same strategy as in 
 Lemma \ref{lem:grad_approx}.

     Thus we obtain that: \begin{equation*}
         \begin{split}
              - &\hat{p}^\top \nabla_p\mathcal L(\mE^1, p_t) \\
              & \geq \left(1 - \frac{\mu}{2} - \eps_4 \right) \hbE\left[ g(X, y)\sum_{j \in  \overline{\cS_X(p_t)}} ( \widehat{a_{i_*^X}}(X) -  \widehat{a_j}(X) ) q_{i_*^X}(X) q_j(X) (\gamma_{i_*^X}(X) - \gamma_j(X))  )  \right] \\
              & \geq \left(1 - \frac{\mu}{2} - \eps_4 \right) \frac{1}{n} \sum_{X \in \cX_0} g(X, y)\sum_{j \in  \cS^2_X(p_t)} ( \widehat{a_{i_*^X}}(X) -  \widehat{a_j}(X) ) q_{i_*^X}(X) q_j(X) (\gamma_{i_*^X}(X) - \gamma_j(X))),   
         \end{split}
     \end{equation*} where in the last inequality we use that all the summand are non-negative.

%
 Note that 
 \begin{equation*}
        \widehat{a_{i_*^X}}(X) -  \widehat{a_j}(X)  = 1, \qquad \overline{a_{i_*^X}}(X) -  \overline{a_j}(X) \leq 1- \mu +\eps_1, \qquad \forall X \in \cX_0.
    \end{equation*}  Thus, it remains to show that
    \begin{equation}
    \label{eq:residue}
    \begin{split}
         \frac{\mu}{4 n}  & \sum_{X \in \cX_0} g(X, y)\sum_{j \in  \cS^2_X(p_t)} ( \widehat{a_{i_*^X}}(X) -  \widehat{a_j}(X) ) q_{i_*^X}(X) q_j(X) (\gamma_{i_*^X}(X) - \gamma_j(X)) \\
         \geq & (1 + \eps_2) (1+\eps_3) \frac{1}{n} \sum_{X \in \cX_\circ}  
    g(X, y) \sum_{i \in \cS_X(p_t): i \neq i_*^X} \sum_{j \in \cS^2_X(p_t)} ( \overline{a_{i}}(X) -  \overline{a_j}(X) ) q_{i}(X) q_j(X) |\gamma_{i}(X) - \gamma_j(X)|
    \end{split}
    \end{equation}

    We have that\begin{equation*}
        \widehat{a_{i_*^X}}(X) -  \widehat{a_j}(X)  \geq 1, \qquad \overline{a_{i}}(X) -  \overline{a_j}(X) \leq 1- \mu +\eps_1, \qquad \forall \,i \in \cS_X(p_t), \quad \forall \,X \in \cX_0.
    \end{equation*} Furthermore, \begin{equation*}
        | \gamma_i(X) - \gamma_j(X) | \leq 22 \eta_0 d ^{-1/4},\quad  \forall \,i \neq i_*^X, \qquad \gamma_{i_*^X}(X) - \gamma_j(X) \geq \frac{\eta_0}{4 nT }.
    \end{equation*} As
        $d \geq \left( \frac{2816 Tn^2  (1+2\eta_0)}{\mu }   \right)^4 \geq  \left( \frac{176 n^2 T(2- \mu)}{\mu}   \right)^4$,
  \eqref{eq:residue}  holds and the proof is complete.
\end{proof}

\section{Details of numerical experiments}\label{appendix:hyperparams}

For all numerical simulations, we use the AdamW optimizer from \texttt{torch.optim}, and we reduce the learning rate in a multiplicative fashion by a factor $\gamma = 0.1$ at epochs $100$ and $200$, i.e.,
$$
\mathrm{LR}_{\mathrm{new}} = \mathrm{LR}_{\mathrm{old}} \cdot \gamma.
$$
We adhere to the batch size of $128$ and fix the embedding dimension to $2048$.

\textbf{IMDB and Yelp datasets.} The hyperparameters \emph{do not} differ between the two-layer model and the one-layer model. We set the number of training epochs to $500$, the learning rate to $0.01$, and the weight decay to $10^{-8}$. 

\textbf{Synthetic data.}  We set the number of training epochs to $196$, the learning rate to $10^{-4}$, and the weight decay to $10^{-4}$.

\end{document}